%% file: ms.tex
\documentclass{article}

\usepackage[accepted]{icml2019nc}

\usepackage{microtype, graphicx, booktabs,verbatim, amsmath, amssymb, subfig, bbm, mathrsfs, amsthm, url,adjustbox,mathtools}
\usepackage[normalem]{ulem}
\usepackage[framemethod=TikZ]{mdframed}

\input{commands.tex}

\newcommand{\namecite}[1]{\citeauthor{#1}~(\citeyear{#1})}

\newcommand{\algname}{Covering options} 

\icmltitlerunning{Discovering Options for Exploration by Minimizing Cover Time}

\begin{document}

\twocolumn[
\icmltitle{Discovering Options for Exploration by Minimizing Cover Time}
\icmlkeywords{Reinforcement Learning}
\begin{icmlauthorlist}
\icmlauthor{Yuu Jinnai}{b}
\icmlauthor{Jee Won Park}{b}
\icmlauthor{David Abel}{b}
\icmlauthor{George Konidaris}{b}
\end{icmlauthorlist}

\icmlaffiliation{b}{Brown University, Providence, RI, United States}

\icmlcorrespondingauthor{Yuu Jinnai}{yuu\_jinnai@brown.edu}

\vskip 0.3in
]

\printAffiliationsAndNotice{}


\begin{abstract}
    One of the main challenges in reinforcement learning is solving tasks with sparse reward. 
    We show that the difficulty of discovering a distant rewarding state in an MDP is bounded by the expected cover time of a random walk over the graph induced by the MDP's transition dynamics.
    We therefore propose to accelerate exploration by constructing options that minimize cover time.
    The proposed algorithm finds an option which provably diminishes the expected number of steps to visit every state in the state space by a uniform random walk.
    We show empirically that the proposed algorithm improves the learning time 
    in several domains with sparse rewards.
\end{abstract}

\section{Introduction}

A major challenge in reinforcement learning is how an agent should explore its environment when the reward signal is sparse.
\namecite{machado2017laplacian} have addressed the sparse reward problem through the construction of temporally extended actions, commonly formalised as options~\cite{sutton1998reinforcement}.
Previous approaches develop techniques that lead to improved exploration in sparse reward problems, but it is still an open question as to how to explore in a near-optimal way in these tasks.


We introduce an option discovery method that explicitly aims to improve exploration in sparse reward domains by
minimizing the expected number of steps to reach an unknown rewarding state.
First, we model the behavior of an agent early in its learning process (that is, before observing the reward signal) as a uniform random walk over the graph induced by the MDP's transition dynamics.
We show that minimizing the graph \textit{cover time}---the number of steps required for a random walk to visit every state \cite{broder1989bounds}---reduces the expected number of steps required to reach an unknown rewarding state. 

We then introduce a polynomial time algorithm to find a set of options guaranteed to reduce the expected cover time using the transition function either given to or learned by the agent. 
Finding a set of edges that minimizes expected cover time is an extremely hard combinatorial optimization problem \cite{braess1968paradoxon,braess2005paradox}. Thus, our algorithm instead seeks to minimize the upper bound of the expected cover time given as a function of the algebraic connectivity of the graph Laplacian \cite{fiedler1973algebraic,broder1989bounds,chung1996spectral} using the heuristic method by \namecite{ghosh2006growing} that improves the upper bound of the expected cover time of a uniform random walk. 

Finally, we evaluate our option discovery algorithm in six discrete benchmark domains where the agent is given
the true MDP graph but must learn the location of the reward online. Our empirical results in toy problems demonstrate that the approach outperforms previous state-of-the-art methods. 

\section{Background}


\subsection{Reinforcement Learning}

Reinforcement learning defines the problem of learning a policy that maximizes the total expected reward obtained by an agent interacting with an environment.
The environment is often modeled as a Markov Decision Process (MDP) \cite{puterman2014markov}. An MDP is a five tuple $(\mc{S}, \mc{A}, T, R, \gamma)$, where $\mc{S}$ is a set of states, $\mc{A}$ is a set of actions, $T: \mc{S} \times \mc{A} \times \mc{S} \rightarrow [0, 1]$ is a state transition function, $R: \mc{S} \times \mc{A} \rightarrow \mathbb{R}$ is a reward function, $\gamma \rightarrow [0, 1]$ is a discount factor.

The agent selects actions according to a policy $\pi: \mc{S} \times \mc{A} \rightarrow [0, 1]$ mapping states to actions.
The expected total discounted reward from state $s$ following a policy $\pi$ is the value of the state:
\begin{equation}
    V^\pi(s) = R(s, \pi(s)) + \gamma \sum_{s' \in \mc{S}} T(s, \pi(s), s') V^\pi(s').
\end{equation}

This function is called a \textit{value function}. 
The action-value function of a policy is an expected total discount reward received by executing an action $a$ and then follow policy $\pi$:

\begin{equation}
    Q^\pi(s, a) = R(s, a) + \gamma \sum_{s' \in \mc{S}} T(s, a, s') V^\pi(s'). 
\end{equation}

The goal of the agent is to learn an optimal policy $\pi^*$ which maximizes the total discounted reward: $\pi^* = \argmax_\pi V^\pi$, 
with corresponding optimal value functions $V^* = \max_\pi V^\pi$ and $Q^* = \max_\pi Q^\pi$.


A state-transition in an MDP by a stationary policy $\pi$ can be modeled as a Markov chain $\{X_t\}$ where $P(X_{t+1} | X_{t}) = \sum_{a \in A} \pi(a|s) T(s, a, s') |_{X_{t+1} = s', X_{t} = s}$.
A state-transition graph $G = (V, E)$ of an MDP is a graph with nodes representing the states in the MDP and the edges representing state adjacency in the MDP. More precisely, $V = S$, $e(s, s') \in E \; \text{iff} \; \exists a T(s, a, s') > 0 \vee T(s', a, s) > 0$. 
An adjacency matrix represents a graph with a square matrix of size $|S| \times |S|$ with $(i, j)$-value being 1 if $e(s_i, s_j) \in E$ and 0 otherwise.

\subsection{Options}

Temporally extended actions offer great potential for mitigating the difficulty of solving difficult MDPs in planning and reinforcement learning \cite{sutton1999between}.
We use one such framework, the \textit{options framework} \cite{sutton1999between}, which defines a temporally-extended
action as follows.
\ddef{option}{An option $o$ is defined by a triple: $(\mc{I}, \pi, \beta)$ where:
    \begin{itemize}
        \item $\mc{I} \subseteq \mc{S}$ is a set of states where the option can initiate,
        \item $\pi : \mc{S} \ra \Pr(\mc{A})$ is a policy,
        \item $\beta : \mc{S} \ra [0, 1]$, is a termination condition.
    \end{itemize}
}

Many previous approaches propose methods to generate options based on heuristics and demonstrate the effectiveness in experimental evaluation \cite{iba1989heuristic,mcgovern2001automatic,menache2002q,stolle2002learning,Simsek04,csimcsek2009skill,konidaris2009skill,machado2017eigenoption,eysenbach2018diversity}.

As the options framework is general and difficult to analyze, we focus on \textit{point options} \cite{jinnai2018finding},
a simple subclass of options where both the initiation set and termination condition consist of a single state. 
\ddef{Point option}{
    A {\bf point option} is any option whose initiation set and termination set are each true for exactly one state each:
\begin{align}
        |\{s \in \mc{S} : \mc{I}(s) = 1\}| &= 1, \\
        |\{s \in \mc{S} : \beta(s) > 0\}| &= 1, \\
        |\{s \in \mc{S} : \beta(s) = 1\}| &= 1.
\end{align}
}
Adding a point
option corresponds to inserting a single edge into the graph induced by the MDP dynamics. We refer to the state with $\beta(s) = 1$ as the subgoal state.

Point options are a useful subclass to consider for several reasons. A point option is a simple model of a temporally extended action whose effect on the MDP graph is easy to specify, and 
whose policy can often be efficiently computed. Moreover, any  option with a single termination state can be represented as a collection of point options.

\section{Cover Time}
\label{sec:cover-time}


In this section, we model the behavior of the agent {\it at the first episode} as a random walk induced by a fixed stationary policy, and show an upper bound to {\it the expected cover time} of the random walk.
We model the behavior of a fixed stationary policy for two reasons.
First, it is a reasonable model for an agent with no prior knowledge of the task. Second, it serves as a {\it worst-case} analysis: it is reasonable to assume that most of the cases efficient exploration algorithms such as UCRL \cite{ortner2007logarithmic,jaksch2010near} explore faster than a fixed stationary policy. Thus the upper bound we show for the expected cover time is applicable to other algorithms.

Intuitively, the expected cover time is the time required for a random walk to visit all the vertices in a graph \cite{broder1989bounds}. To define it formally, we first define the hitting time of any discrete Markov chain $\{X_t\}$. Let us assume this Markov chain has the state space of $V$, the vertices of graph $G$.
The hitting time $H_{ij}$, where $i,j \in V$, is 
\begin{equation}
    H_{ij} = \inf\big\{t: X_t = j | X_0 = i\big\}.
\end{equation}
In other words, $H_{ij}$ is the greatest lower bound on the number of time step $t$ required to reach 
state $j$ after starting at state $i$. Cover time starting from state $i$ is defined as:
\begin{equation}
   C_i = \max_{j \in V } H_{ij},
\end{equation}
and the expectation of cover time, $\bE[C(G)]$, is the expected cover time of trajectories induced by the random walk, maximized over the starting states \cite{broder1989bounds}.

As such, the expected cover time bounds how likely a random walk leads to a rewarding state. 

\begin{theorem}
    Assume a stochastic shortest path problem to reach a goal $g$ where a non-positive reward $r_c \leq 0$ is given for non-goal states and $\gamma = 1$.
    Let $P$ be a random walk transition matrix: $P(s, s') = \sum_{a \in A} \pi(s) T(s, a, s')$:
    \begin{align*}
        \forall g: V_g^\pi(s) 
                            &\geq r_c \bE[C(G)],
    \end{align*}
    where $C(G) = \max_{s \in S} C_s(G)$ and $C_s(G)$ is a cover time of a transition matrix $P$ starting from state $s$.
\end{theorem}
\begin{proof}
    The value of state $s$ is $r_c$ times the expected number of steps to reach the goal state. Thus, 
    \begin{align*}
        V_g^\pi(s) &= r_c \bE[ H_{sg} ] \\
                &\geq r_c \bE[ \max_{s' \in S} H_{ss'} ] \\
                &= r_c \bE[ C_s(G) ] \\
                &\geq r_c \bE[ C(G) ]
    \end{align*}
\end{proof}

The theorem suggests that {\it the smaller the expected cover time is, the easier the exploration tends to be}. Now the question is how to reduce the expected cover time of the random walk without prior reward information of the task.

Let $P$ be a random walk induced by a fixed policy $\pi$ in an MDP.
\namecite{broder1989bounds} showed that the expected cover time $\bE[C(G)]$ of a random walk $P$ can be bounded using the second largest eigenvalue $\lambda_{k-1}(P)$:
\begin{equation}
\label{eq:bound}
    \bE[C(G)] \leq \frac{n^2 \ln n}{1 - \lambda_{k-1}(P)} (1 + o(1)),
\end{equation}
where $n = |V|$ and $k$ is the number of eigenvalues. 

The normalized graph Laplacian of an unweighted undirected graph is defined as \cite{chung1996spectral}:
\begin{equation}
    \mathcal{L} = I - T^{-1/2} A T^{-1/2},
\end{equation}
where $I$ is an identity matrix.
The random walk matrix can be written in terms of the Laplacian:
\begin{equation}
    P = T^{-1} A = T^{-1/2} (I - \mathcal{L}) T^{1/2}.
\end{equation}

Because $P$ and $I - \mathcal{L}$ are similar matrices, they have the same eigenvalues and eigenvectors. Thus, $\lambda_{k-1}(P) = 1 - \lambda_2(\mathcal{L})$, where $\lambda_2(\mathcal{L})$ is the second smallest eigenvalue of $\mathcal{L}$. Thus, from Equation \ref{eq:bound},

\begin{equation}
    \bE[C(G)] \leq \frac{n^2 \ln n}{\lambda_{2}(\mathcal{L})} (1 + o(1)).
\end{equation}

Thus, {\it the larger the $\lambda_{2}(\mathcal{L})$ is, the smaller upper bound of the expected cover time is}.




The second smallest eigenvalue of $\mathcal{L}$ is known as the {\it algebraic connectivity} of the graph and its corresponding eigenvector is called {\it Fiedler vector}.
There are several operations we can apply to the graph to increase the algebraic connectivity.
First, adding nodes to the graph can increase the algebraic connectivity. However, this increases the number of nodes $n$, and thus the cover time does not always improve.
Second, we can rewire edges in the graph. However, rewiring edges is undesirable as it amounts to removing primitive actions from the MDP
which may damage the agent's ability to optimally solve the MDP. 
Third, we can add edges to the graph, which in the reinforcement learning setting amounts to adding options to the agent. This strategy preserves optimality as it does not remove any primitive actions.
Therefore, adding edges (i.e. options) is a reliable way to reduce the cover time without potentially sacrificing optimality.

As far as we are aware, we are the first to introduce the concept of the cover time to reinforcement learning.

\subsection{Empirical Evaluation}

\begin{figure*}
    \centering
    \subfloat[($\lambda_2$) $\sim$ (expected cover time)]{\label{fig:connectivity-covertime} \includegraphics[width=0.4\textwidth]{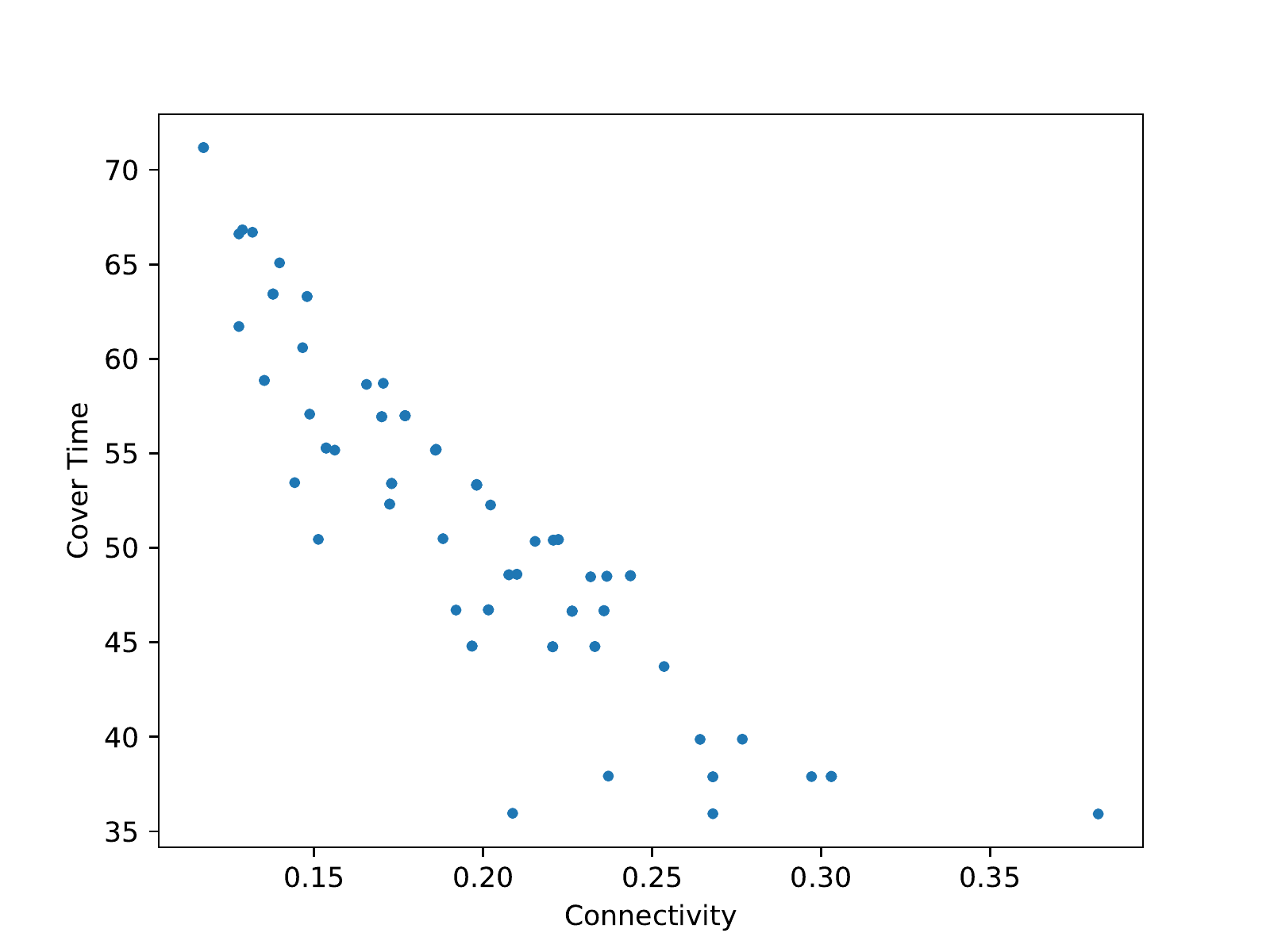}}
    \subfloat[(expected cover time) $\sim$ (cost of the policy)]{\label{fig:covertime-hits} \includegraphics[width=0.4\textwidth]{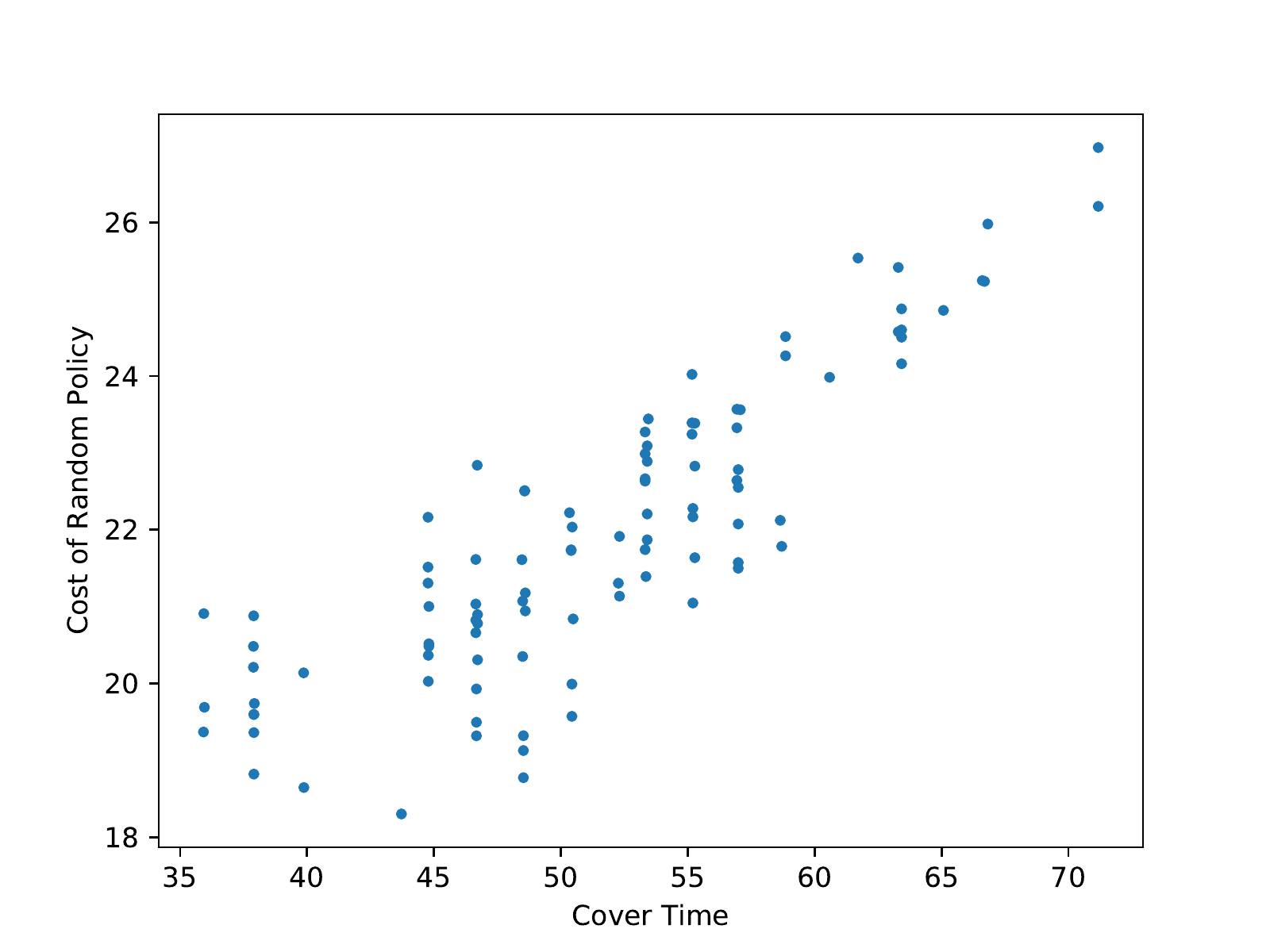}}
    \caption{(a) Relationship between the algebraic connectivity and the expected cover time of a random walk on randomly generated connected graph. The number of states is fixed to 10 and the edge density is fixed to $0.3$, thus the number of edges are equal for all tasks. (b) Relationship between the expected cover time of a random walk and the cost of random policy. The number of states is fixed to 10 and the edge density is fixed to $0.3$, thus the number of edges are equal for all tasks.} 
    \label{fig:covertime}
\end{figure*}

In this section 
we showed that the bigger the algebraic connectivity, the smaller the upper bound of the expected cover time.
Here, we empirically examine (1) the relationship between the algebraic connectivity and the cover time, and (2) the relationship between the cover time and the difficulty of an MDP.

We randomly generated shortest path problems and plotted the relationship between the value of a random policy, the cover time, and the algebraic connectivity of the state-space graph. 

We generated 100 random graphs by the following procedure.
Each graph is connected and has 10 nodes with the edge density  fixed to $0.3$.
To generate a connected graph, we use the following procedure.
First, we start with an empty graph. We pick one node from the existing graph and add an edge to connect to a new node. We follow this procedure for the number of nodes $n$, generating a random tree of size $n$. Then, we pick an edge uniformly randomly from $E^c$ until the edge density reaches the threshold.
We approximated the expected cover time of a random walk on a random graph by sampling 10,000 trajectories induced by the random walk and computing their average cover time.

We generated a shortest path problem by picking an initial state and a goal state randomly for each graph.
The agent can transition to each neighbor with a cost of $1$.

Figure \ref{fig:connectivity-covertime} shows the relationship of the algebraic connectivity and the expected cover time of the random walk induced by a uniform random policy. The result shows that the random walk tends to have smaller expected cover time when the underlying state-transition graph has larger algebraic connectivity.
Figure \ref{fig:covertime-hits} shows the expected cost of a random policy from the initial state to reach the goal state. The cost of a random policy shows a correlation to the cover time.



\section{\algname{}}
\label{sec:algorithm}

We now describe an algorithm to automatically find options that minimize the expected cover time. The algorithm is approximate, since 
the problem of finding such a set of options is computationally difficulty; it is thought to be NP-hard, but noboby so far has proven that. Even a good solution is hard to find due to the Braess's paradox \cite{braess1968paradoxon,braess2005paradox} which states that 
the expected cover time does not monotonically decrease as edges are added to the graph.

Thus, 
expected cover time is often minimized indirectly via maximizing algebraic connectivity \cite{fiedler1973algebraic,chung1996spectral}.
The expected cover time is upper bounded by quantity involving the
algebraic connectivity, and by maximizing it the bound can be minimized \cite{broder1989bounds}.
Adding a set of edges to maximize the algebraic connectivity is  NP-hard \cite{mosk2008maximum}, so we use the approximation method by \namecite{ghosh2006growing}. 

The algorithm is as follows:
\begin{enumerate}
    \item Compute the second smallest eigenvalue and its corresponding eigenvector (i.e., the Fiedler vector) of the Laplacian  of the state transition graph $G$.
    \item Let $v_i$ and $v_j$ be the state with largest and smallest value in the eigenvector respectively. Generate two point options; one with $\mc{I}=\{v_i\}$ and $\beta=\{v_j\}$ and the other one with $\mc{I}=\{v_j\}$ and $\beta=\{v_i\}$. Each option policy is the optimal path from the initial state to the termination state.
    \item Set $G \leftarrow G \cup \{(v_i, v_j)\}$ and repeat the process until the number of options reaches $k$.
\end{enumerate}

Intuitively, the algebraic connectivity represents how tightly the graph is connected. The Fiedler vector is an embedding of a graph to a line (single real value) where nodes connected by an edge tend to be placed close by.  A pair of nodes with the maximum and minimum value in the Fiedler vector are the most distant nodes in the embedding space. Our method greedily connects the two most distant nodes in the embedding. It is known that this operation greedily maximizes the algebraic connectivity to a first order approximation \cite{ghosh2006growing}.

The state transition graph $G$ must be given to or learned by the agent. We assume that the graph is strongly connected, so every state is reachable from every other state, and also that the graph is undirected.

The algorithm finds the edge which connects two nodes with largest and smallest value in the Fiedler vector ($v_i$ and $v_j$). \namecite{ghosh2006growing} proved that adding this edge to the original graph maximizes the algebraic connectivity greedily.
Thus, our algorithm generates options which maximize the algebraic connectivity, which in turn minimizes the upper bound of the expected cover time.
The algorithm is guaranteed to improve the upper bound and the lower bound of the expected cover time:

\begin{theorem}
    Assume that a random walk induced by a policy $\pi$ is a uniform random walk:
    \begin{equation}
    \label{eq:random-walk}
        P(u, v) := \begin{cases}
            1 / d_u  & \text{if $u$ and $v$ are adjacent,}\\
            0 & \text{otherwise},
            \end{cases}
    \end{equation}
    where $d_u$ is the degree of the node $u$.
    Adding two options by the algorithm improves the upper bound of the cover time if the multiplicity of the second smallest eigenvalue is one: 
    \begin{equation}
    \bE[C(G')] \leq \frac{n^2 \ln n}{\lambda_{2}(\mathcal{L}) + F} (1 + o(1)),
    \end{equation}
    where $\bE[C(G')]$ is the expected cover time of the augmented graph, $F = \frac{(v_i - v_j)^2}{6 / (\lambda_3 - \lambda_2) + 3 / 2}$, and $v_i, v_j$ are the maximum and minimum values of the Fiedler vector.
    If the multiplicity of the second smallest eigenvalue is more than one, then adding any single option cannot improve the algebraic connectivity.
\end{theorem}
\begin{proof}
    Assume the multiplicity of the second smallest eigenvalue is one.
    Let $\mathcal{L}'$ be the graph Laplacian of the graph with an edge inserted to $\mathcal{L}$ using the algorithm by \namecite{ghosh2006growing}.
    By adding a single edge, the algebraic connectivity is guaranteed to increase at least by $F$:
    \begin{equation}
        \lambda_2 \geq \lambda_2 + \frac{(v_i - v_j)^2}{6 / (\lambda_3 - \lambda_2) + 3 / 2},
    \end{equation}
    and the upper bound of the cover time is guaranteed to decrease:
    \begin{align*}
        \bE[C(G')] &\leq \frac{n^2 \ln n}{\lambda_{2}} (1 + o(1)) \\
        &\leq \frac{n^2 \ln n}{\lambda_{2} + \frac{(v_i - v_j)^2}{6 / (\lambda_3 - \lambda_2) + 3 / 2}} (1 + o(1)).
    \end{align*}
    
    As $\frac{(v_i - v_j)^2}{6 / (\lambda_3 - \lambda_2) + 3 / 2}$ is positive,
    \begin{equation}
        \frac{n^2 \ln n}{\lambda_{2} + \frac{(v_i - v_j)^2}{6 / (\lambda_3 - \lambda_2) + 3 / 2}} (1 + o(1)) < \frac{n^2 \ln n}{\lambda_{2}} (1 + o(1)),
    \end{equation}
    thus the upper bound is guaranteed to decrease.
    
    Assume the second smallest eigenvalue is more than one. Then, $\lambda_2(\mathcal{L})= \lambda_3(\mathcal{L})$.
    From eigenvalue interlacing \cite{haemers1995interlacing}, for any edge insertion, $\lambda_2(\mathcal{L}) \leq \lambda_2(\mathcal{L'}) \leq \lambda_3(\mathcal{L})$.
    Thus, $\lambda_2(\mathcal{L'}) = \lambda_2(\mathcal{L})$.
\end{proof}


As in the work by \namecite{machado2017laplacian}, our algorithm can be generalized to the function approximation case using an incidence matrix instead of an adjacency matrix.


\subsection{Comparison to Eigenoptions}

\namecite{machado2017laplacian} proposed a method to generate options using the Laplacian eigenvectors.
The proposed algorithm is similar to eigenoptions but different in several aspects.
First and foremost, \algname{} explicitly seeks to speed up the exploration in reinforcement learning by maximizing the algebraic connectivity to improve the upper bounds of the cover time. On the other hand, while the eigenoption also uses the graph Laplacian for option discovery, their method is repurposed from a feature construction method. As such, eigenoptions are constrained to be orthogonal to each other. While this constraint is beneficial for representation learning, it is not helpful for constructing efficient options.
Second, they are computing different optimization problems. The $k$-th covering option is the one minimizing the algebraic connectivity of the graph augmented with $1$ to $k-1$-th options. The $k$-th eigenoption minimizes the algebraic connectivity of the original subject to the constraint that the option has to be orthogonal to $1$ to $k-1$-th options. We did not find analytical results for how the orthogonal constraint can contribute to minimizing the algebraic connectivity or the expected cover time. 
Thrid, \algname{} is fast to compute as it only needs to compute the Fiedler vector. Although computing the whole graph spectrum is a heavy matrix operation, the Fiedler vector can be computed efficiently even for very large graphs \cite{koren2002ace}.

\section{Empirical Evaluation}
\label{sec:evaluation}

We used six MDPs in our empirical study: a 9x9 grid, a four-room gridworld, Taxi, Towers of Hanoi, Parr's maze, and Race Track.
9x9grid, four-room, and Parr's maze \cite{parr1998reinforcement}, a 2-dimensional grid pathfinding problem where the task is to reach a specific location. The agent can move in four directions but cannot cross walls.
The task in Taxi \cite{dietterich2000hierarchical} is to pick-up passengers and sends them to their destination. Only one passenger can ride on the taxi at the same time. 
Towers of Hanoi consists of three pegs of different-size discs sorted in decreasing order of size on one of the pegs. The goal is to move all discs from their initial peg to a goal peg while keeping the constraint that a smaller disc is above a larger one.
In the Race Track task the agent must reach the finish line by driving a car. The car position and the velocity are discrete. The agent can change the horizontal and vertical velocity by +1, -1, or 0 in each step. If the car hits the track boundary, it is moved back to the starting position.

We compared the performance of \algname{}, eigenoptions \cite{machado2017laplacian}, and betweenness options \cite{csimcsek2009skill}. We compare against these methods because they are the state-of-the-art option generation methods which do not require reward information.
\namecite{machado2017laplacian} proposed to generate a set of options which initialize at every state and terminate at the states which have highest/lowest values for each eigenvector.
To make the comparison simple, we consider a point option version of eigenoption method. For $k$-eigenvectors which correspond to the smallest $k$ eigenvalues, we generate a point option from a state with the highest/lowest value to a state the lowest/highest value in the eigenvector.
The point option constructed in this way minimizes the eigenvalue of each corresponding eigenvector.


First, we consider the case where the agent has perfect knowledge of the state-space graph in advance. Then, we consider the case where the agent does not have perfect knowledge but instead is able to sample the state-transition for given amount of steps. Finally, we evaluate an online option generation which discover options while training in the environment.

\subsection{Offline Option Discovery}

\begin{figure}[tb]
    \centering
    \subfloat[\algname{}]{\includegraphics[width=0.23\textwidth]{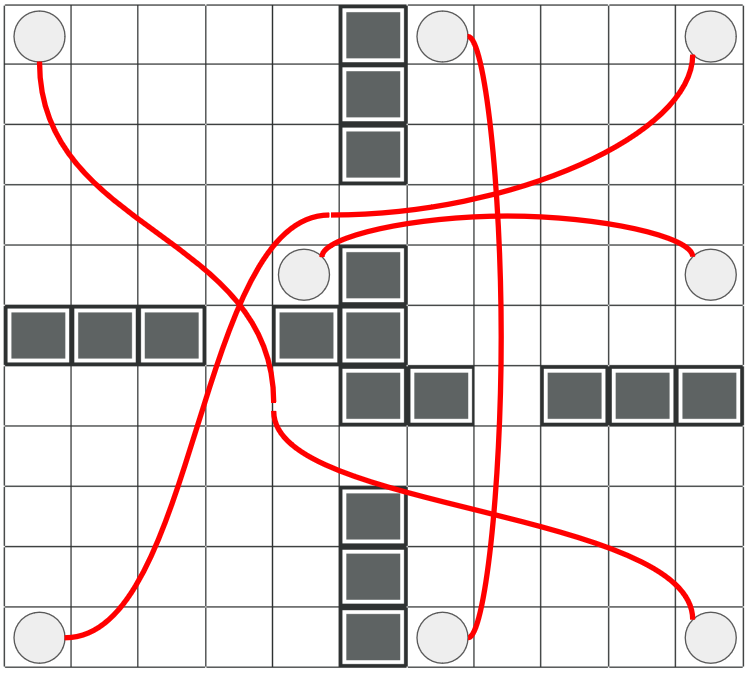}}
    \subfloat[\algname{}]{\includegraphics[width=0.23\textwidth]{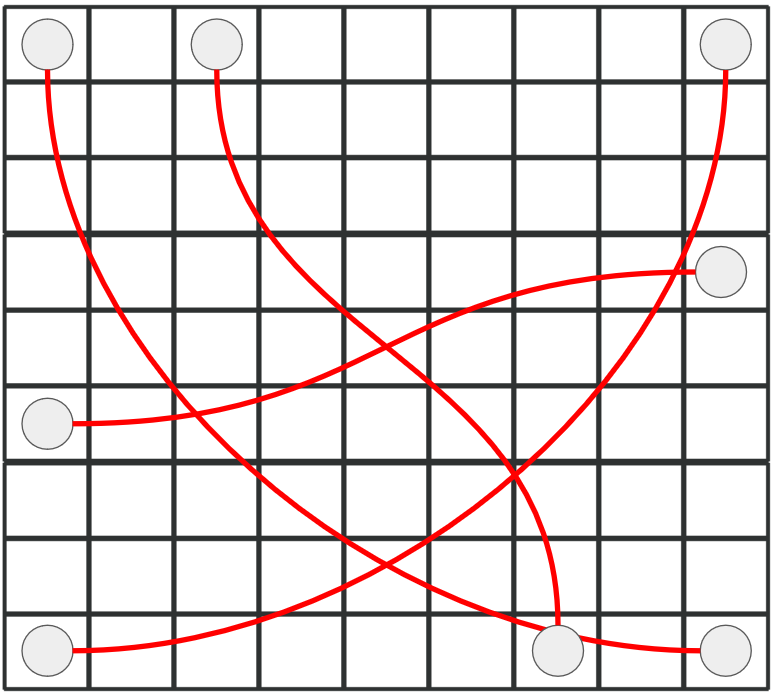}}
    
    \subfloat[Eigenoptions]{\includegraphics[width=0.23\textwidth]{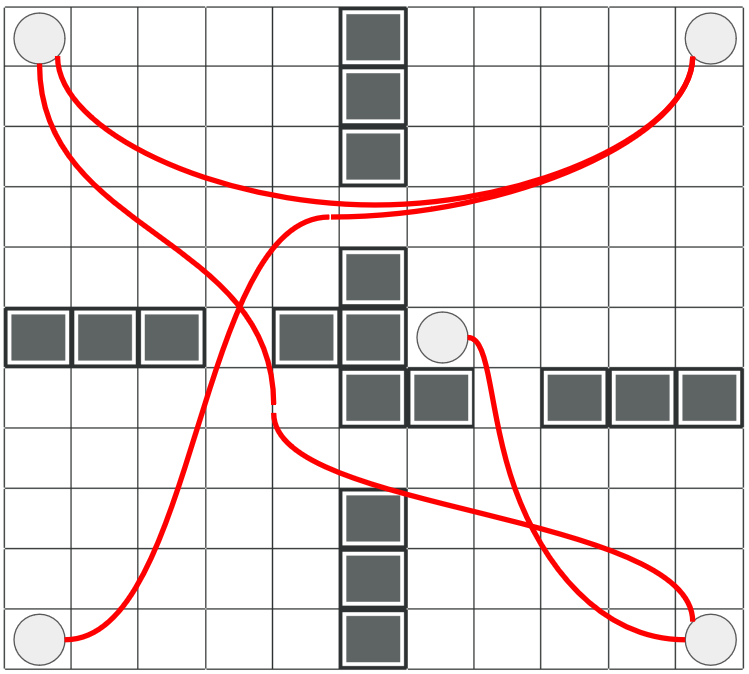}}
    \subfloat[Eigenoptions]{\includegraphics[width=0.23\textwidth]{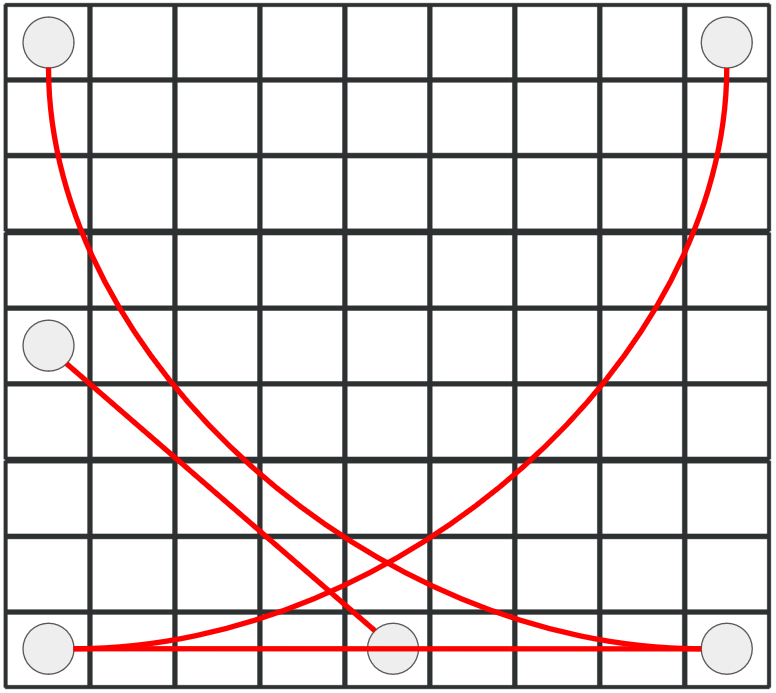}}
    
    \caption{Visualization of Fielder options vs. eigenoptions on four-room domain and 9x9 grid.}
    \label{fig:visuals}
\end{figure}

\begin{table}[htb]
    \centering
    \begin{tabular}{c|c|c}
        four-room        & $\lambda_2$ & Expected Cover Time \\ \hline
        \algname{}      & {\bf 0.065} & {\bf 672.0} \\
        Eigenoptions    & 0.054 & 695.9 \\
        No options      & 0.023 & 1094.8 \\ \hline \hline
        9x9 grid        & $\lambda_2$ & Expected Cover Time \\ \hline
        \algname{}      & {\bf 0.24} & {\bf 258.6} \\
        Eigenoptions    & 0.19 & 261.5 \\
        No options      & 0.12 & 460.5 \\
    \end{tabular}
    \caption{Comparison of the algebraic connectivity and the expected cover time. For \algname{} and eigenoptions we add 8 options.}
    \label{tab:connectivity}
\end{table}

\begin{figure}[htb]
    \centering
    \subfloat[\algname{} (four-room)]{\includegraphics[width=0.23\textwidth]{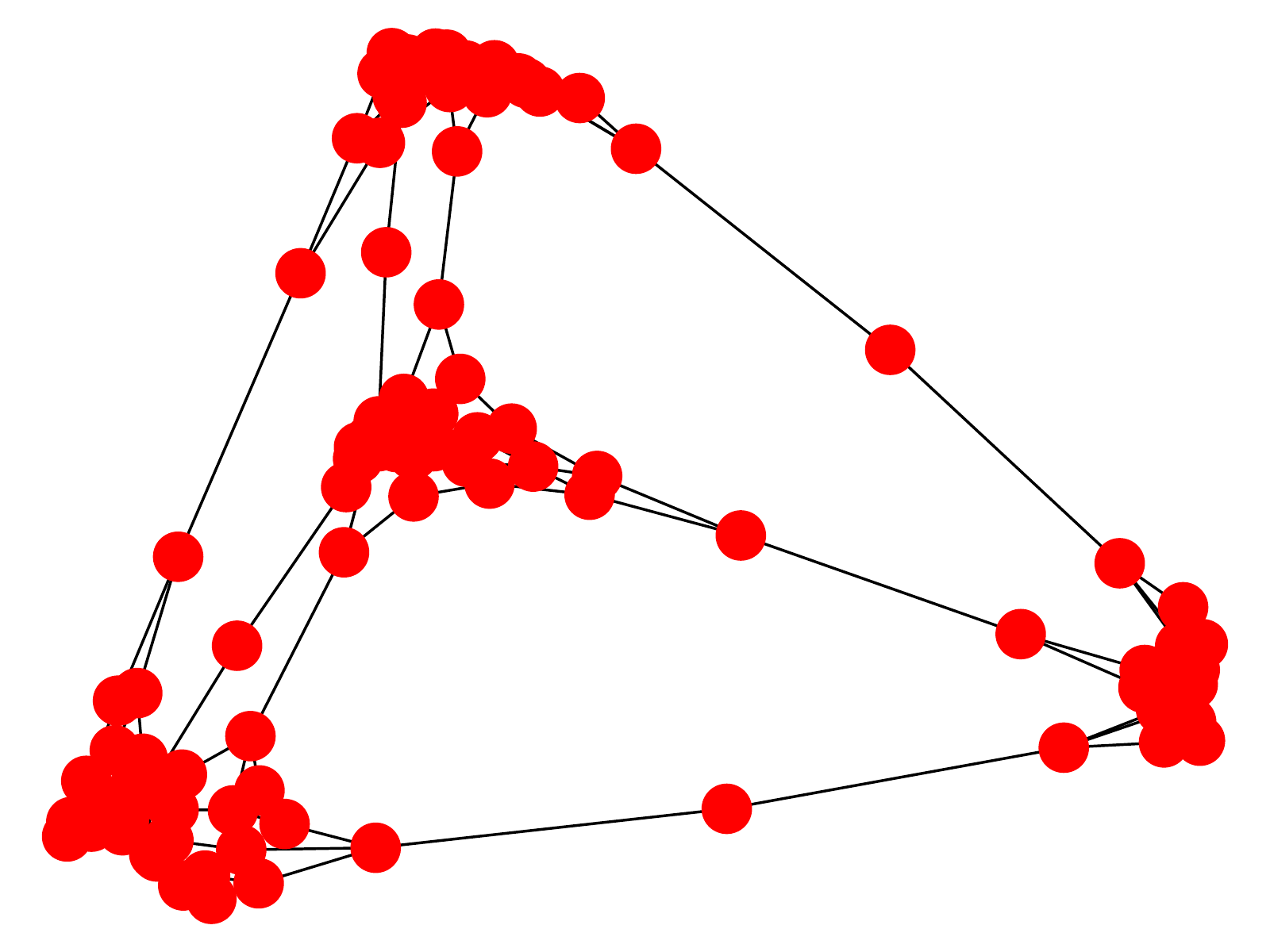}}
    \subfloat[\algname{} (9x9 grid)]{\includegraphics[width=0.23\textwidth]{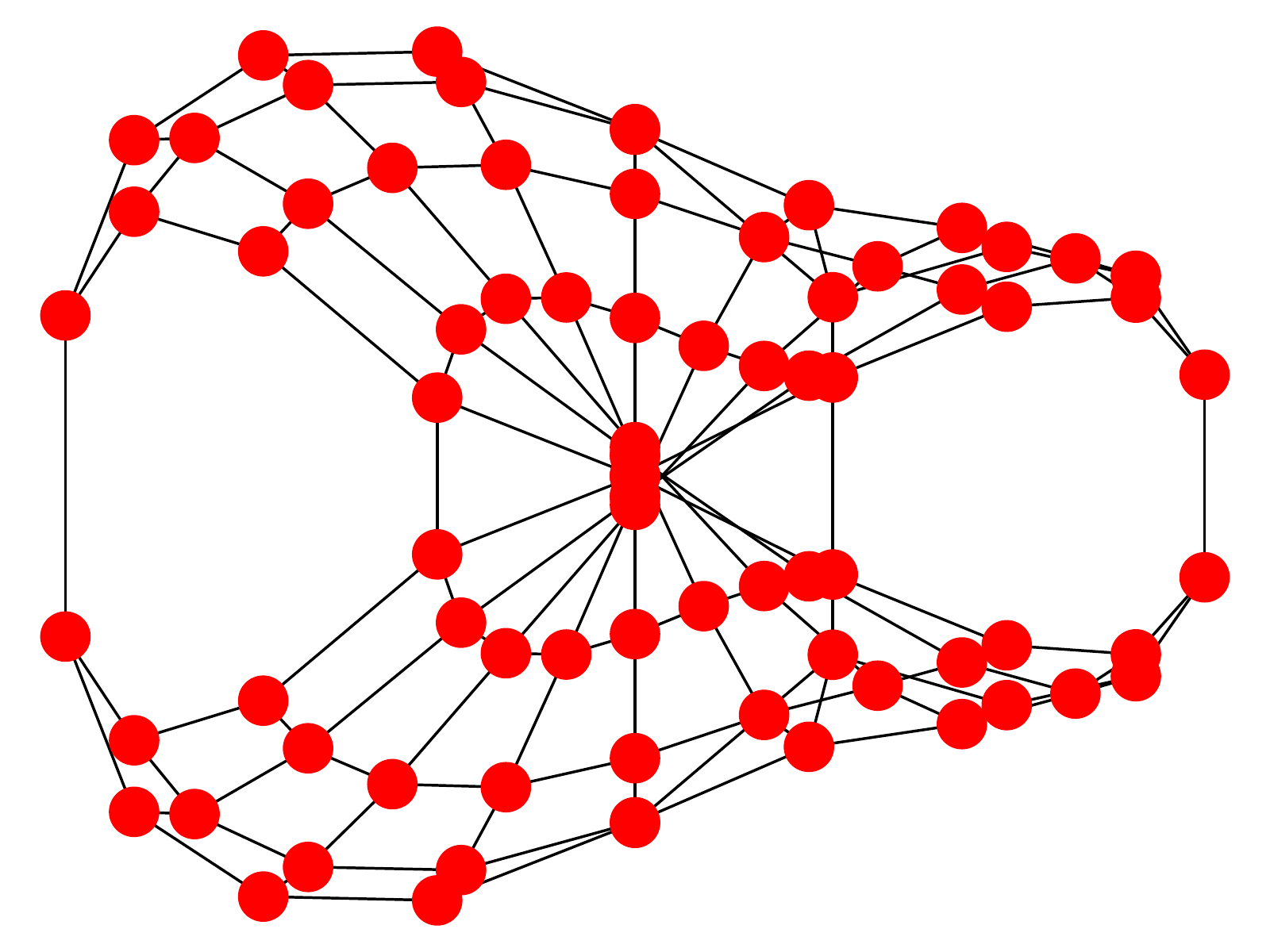}}
    
    \subfloat[Eigenoptions (four-room)]{\includegraphics[width=0.23\textwidth]{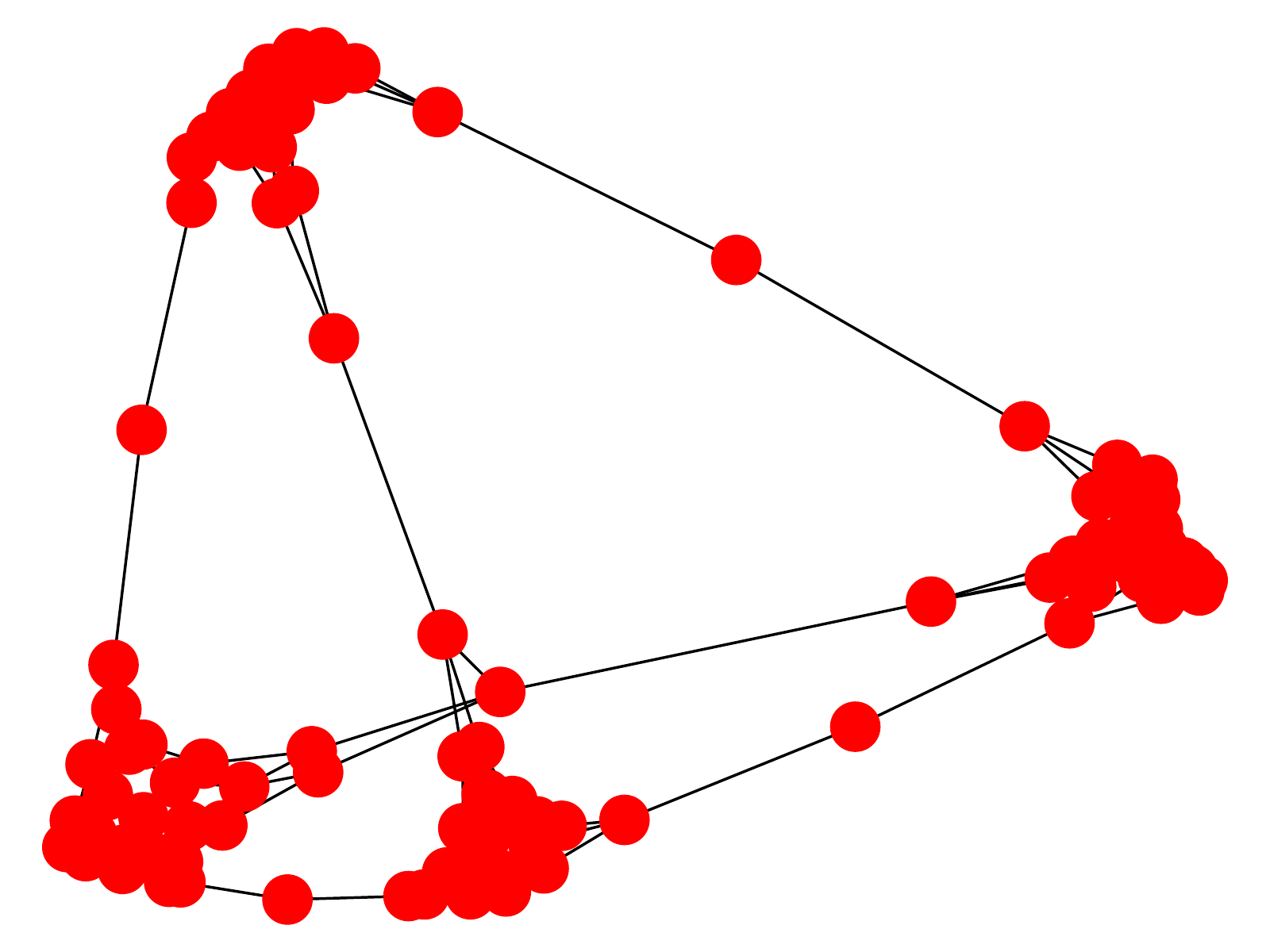}}
    \subfloat[Eigenoptions (9x9 grid)]{\includegraphics[width=0.23\textwidth]{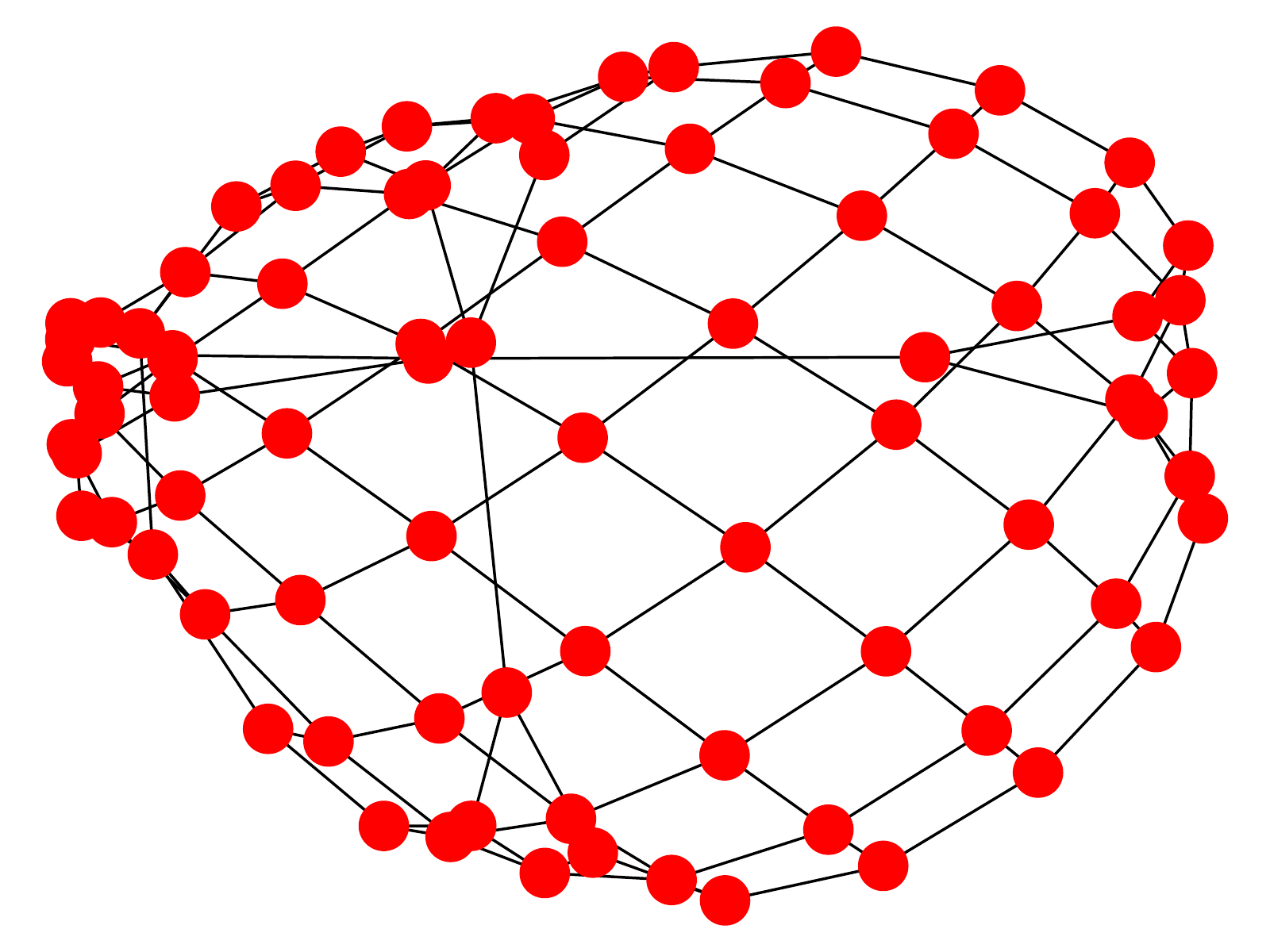}}
    
    \subfloat[No options (four-room)]{\includegraphics[width=0.23\textwidth]{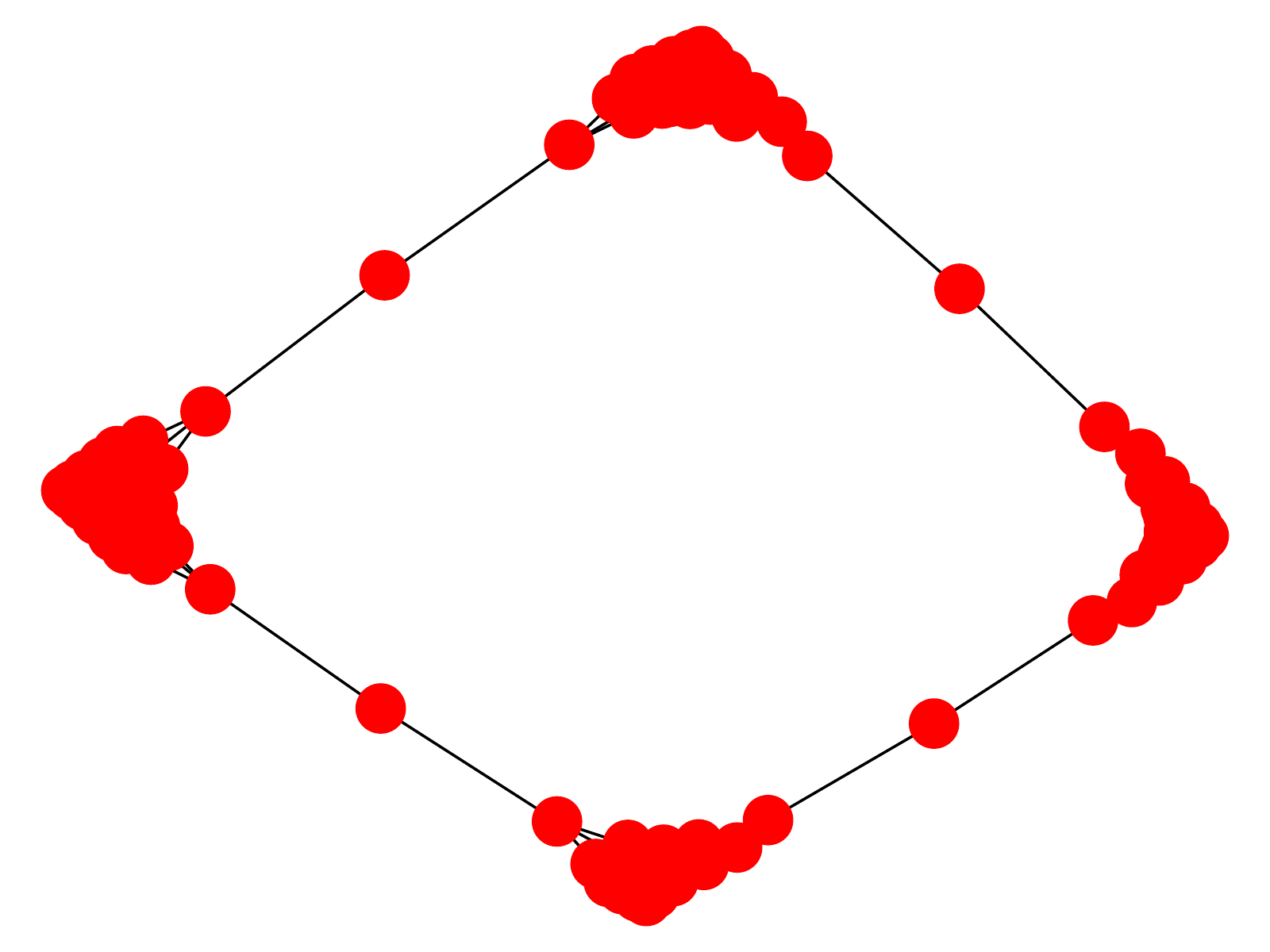}}
    \subfloat[No options (9x9 grid)]{\includegraphics[width=0.23\textwidth]{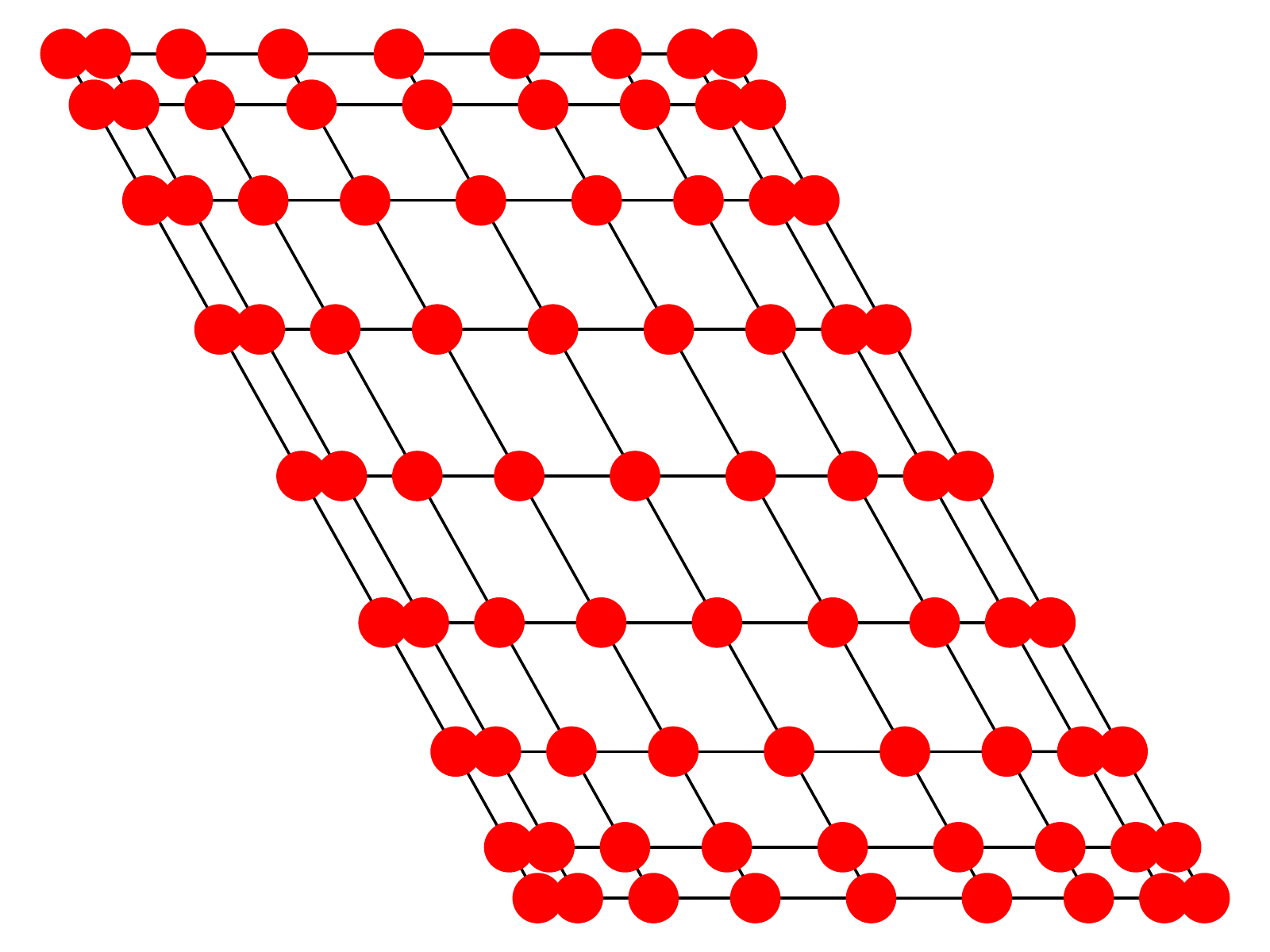}}
    \caption{Spectral graph drawing of the state-transition graph.}
    \label{fig:spectral-drawing}
\end{figure}

Figure \ref{fig:visuals} shows the eight options generated by \algname{}, and Eigenoptions on four-room domain and a 9x9 grid-world domain. Note that there are multiple possible set of options acquired by the algorithm and we showed one of the set of options.

Table \ref{tab:connectivity} shows the algebraic connectivity and  cover time. 
In both domains the \algname{} achieved larger algebraic connectivity and smaller expected cover time than the eigenoptions.
Figure \ref{fig:spectral-drawing} shows the spectral graph drawing \cite{koren2003spectral} of the state-transition graph augmented with the generated options. The spectral graph drawing is a technique to visualize the graph topology using eigenvectors of the graph Laplacian. Each node $n$ in the state-space graph is placed at $(v_2(n), v_3(n))$ in the $(x, y)$-coordinate, where $v_i$ is the $i$-th smallest eigenvector of the graph Laplacian. The figure indicates that the option generation methods are successfully connecting distant states. 

\begin{figure*}[htb]
    \centering
    \subfloat[9x9 grid]{\includegraphics[width=0.33\textwidth]{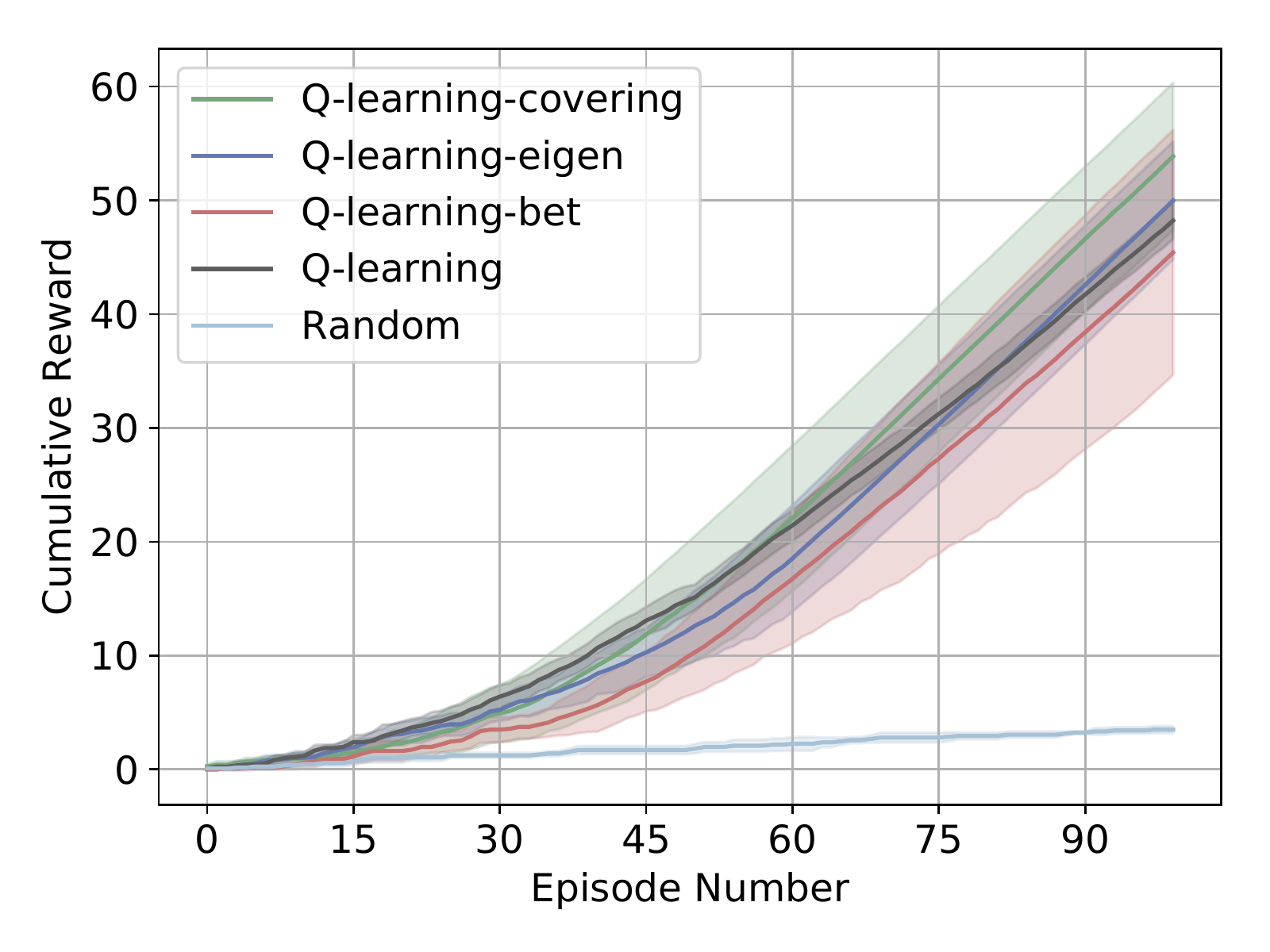}}
    \subfloat[four-room]{\includegraphics[width=0.33\textwidth]{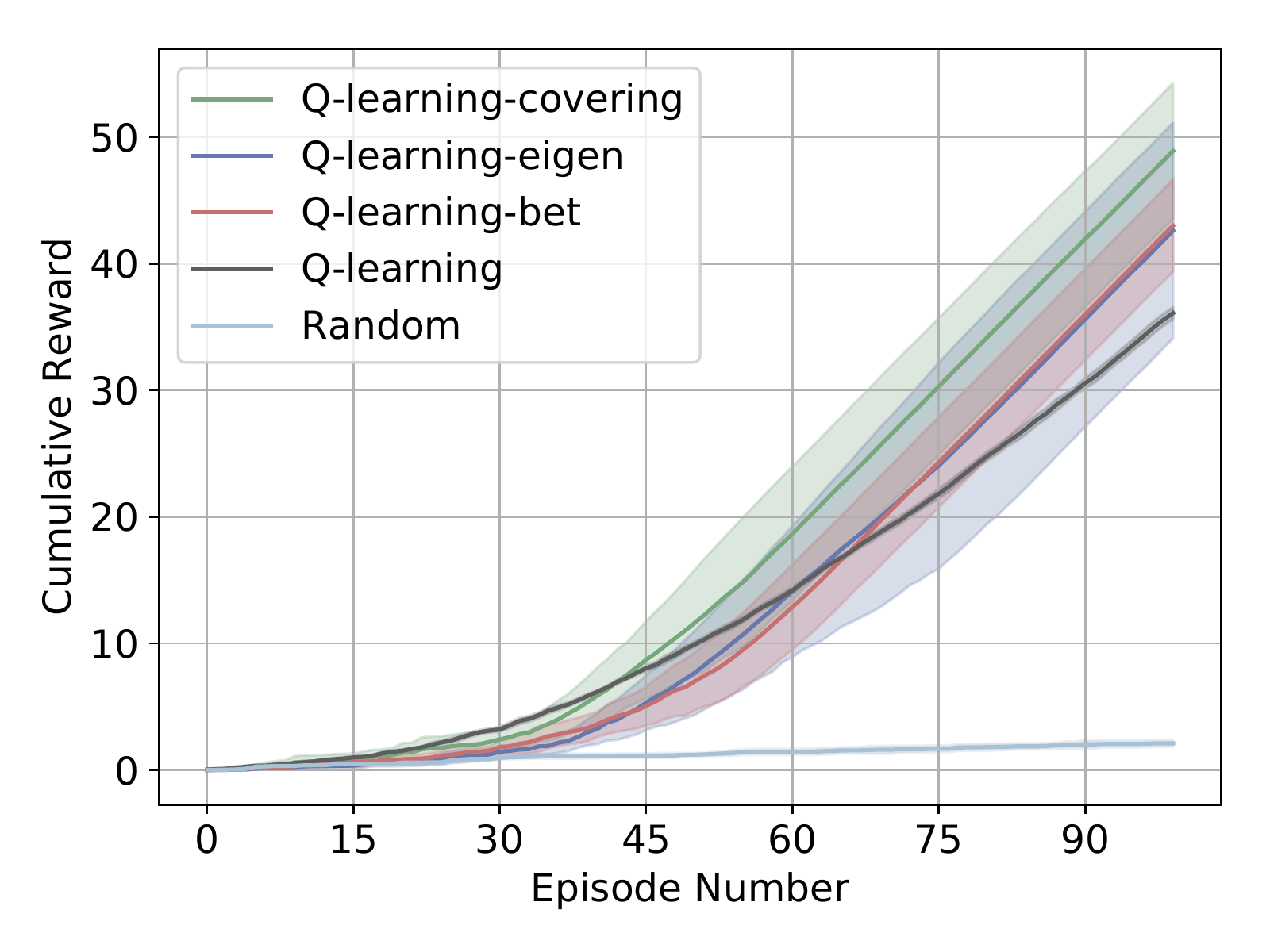}}
    \subfloat[Towers of Hanoi]{\includegraphics[width=0.33\textwidth]{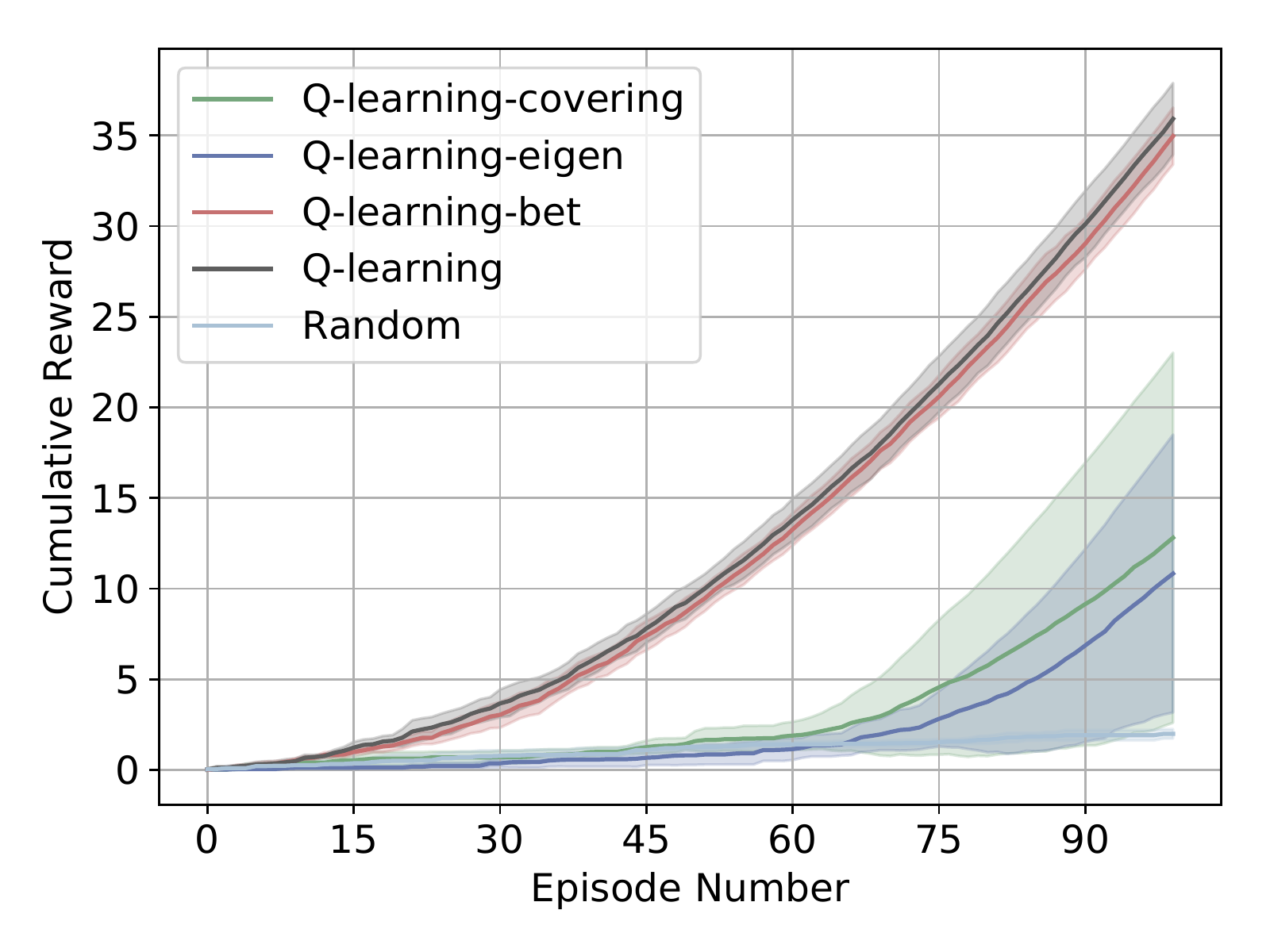}}
    
    \subfloat[Taxi]{\includegraphics[width=0.33\textwidth]{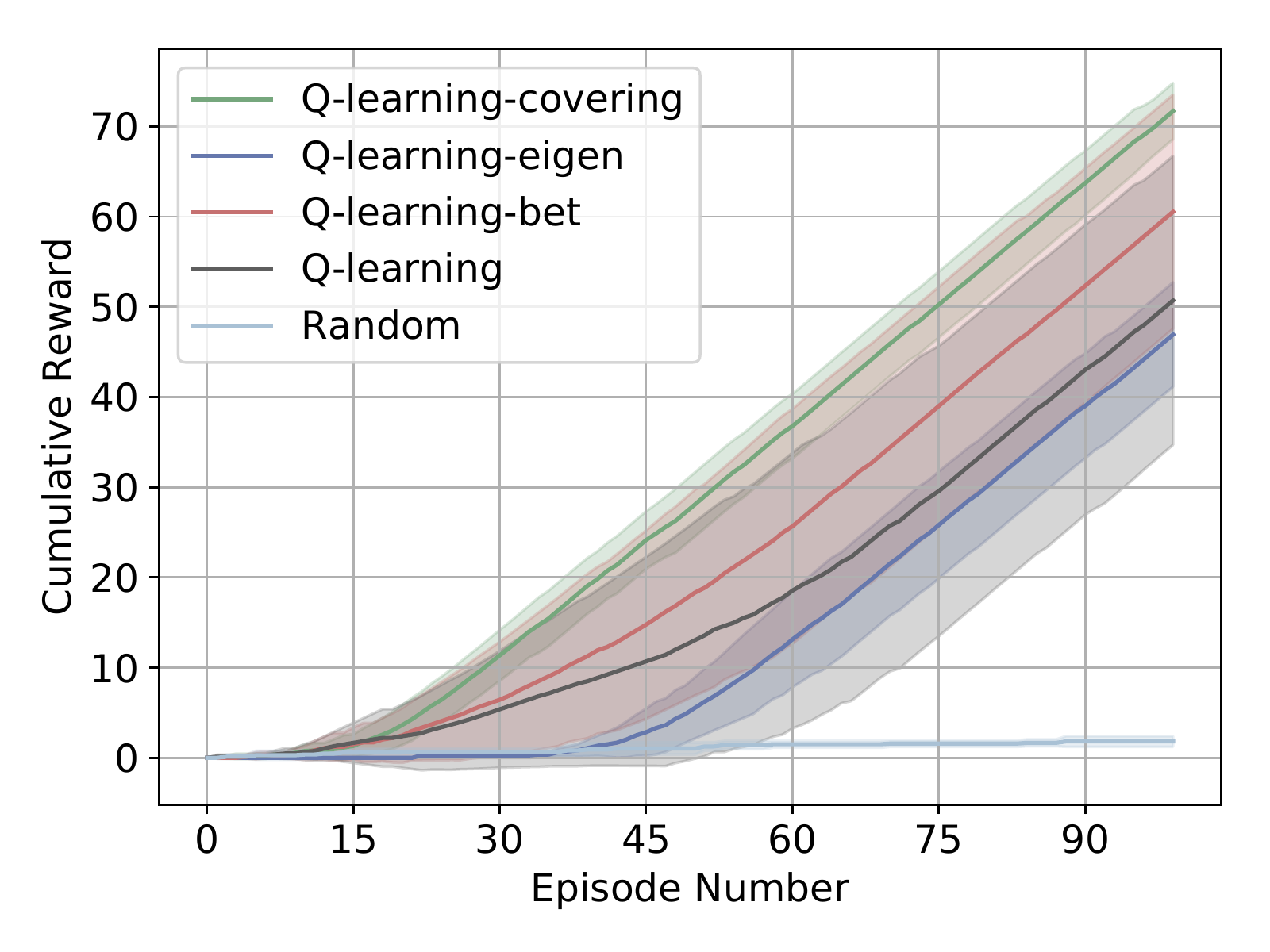}}
    \subfloat[Parr's Maze]{\includegraphics[width=0.33\textwidth]{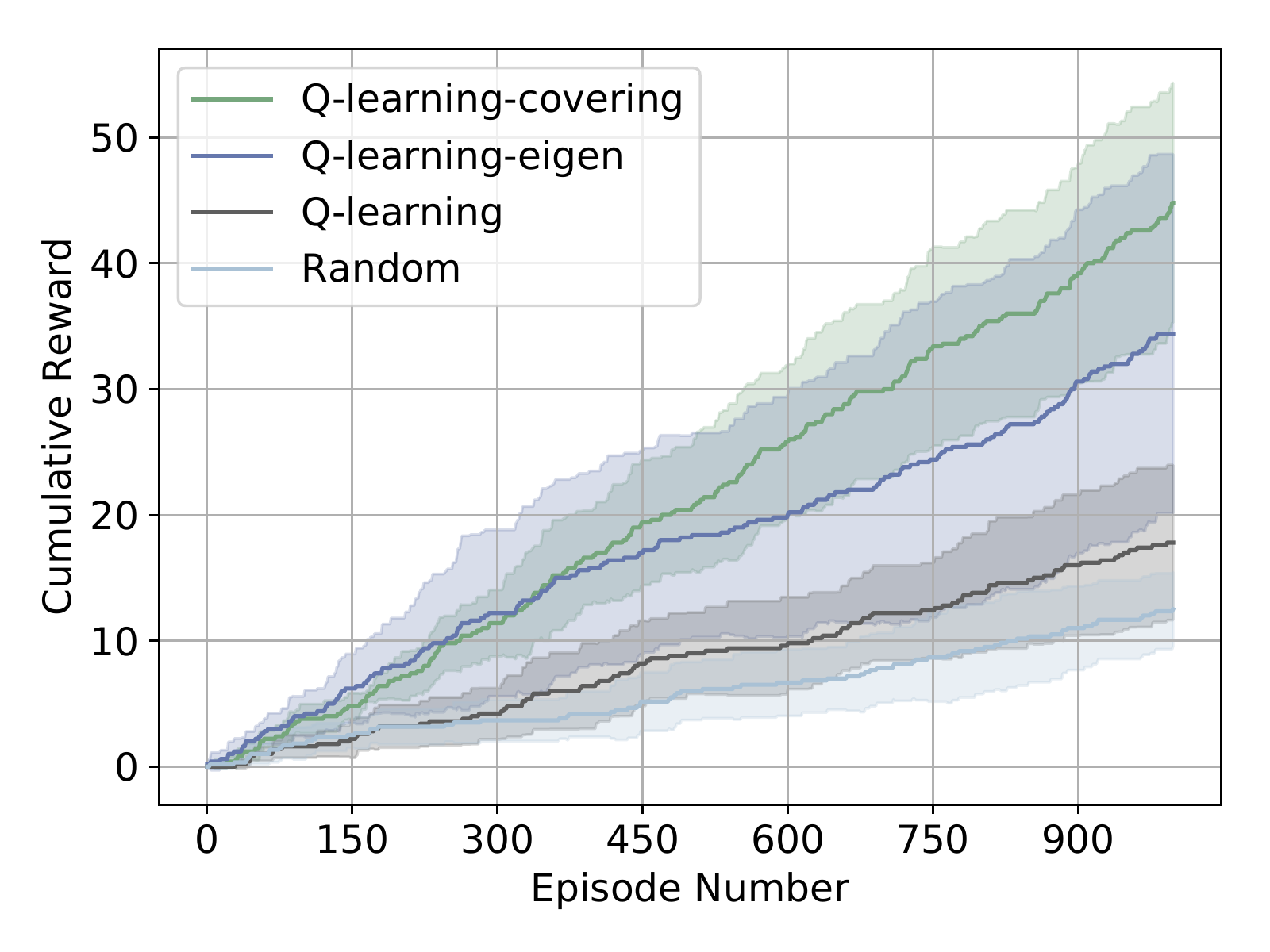}}
    \subfloat[Race Track ]{\includegraphics[width=0.33\textwidth]{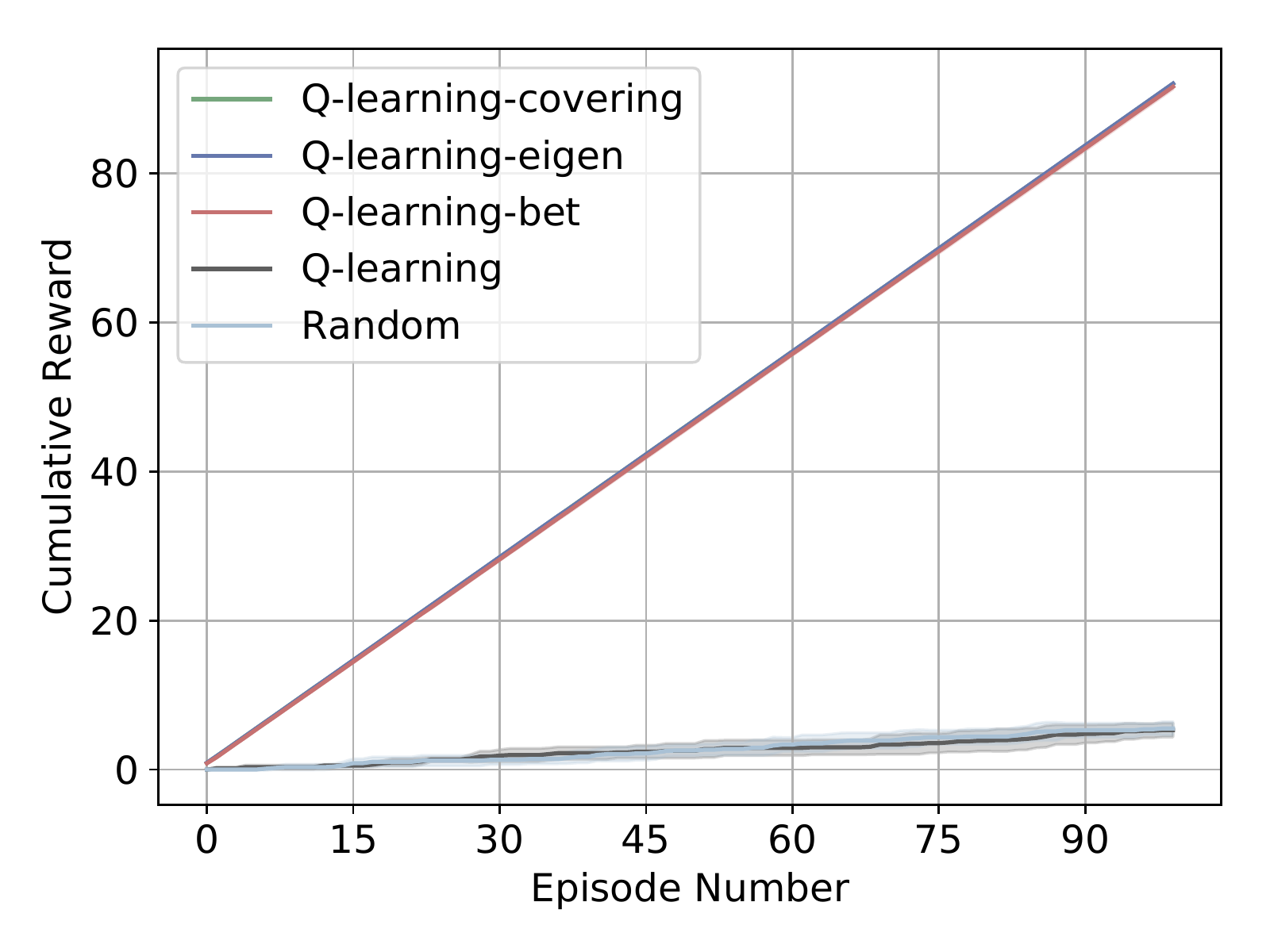}}
    
    \caption{Comparison of performance with different option generation methods. Options are generated offline from the adjacency matrix for 9x9grid, four-room, Towers of Hanoi, and Taxi. Options are generated offline from an incidence matrix for Parr's maze and Race Track. Reward information is not used for generating options.}
    \label{fig:adjacency}
\end{figure*}

We now evaluate the utility of each discovered options for speeding up learning.
We used Q-learning \cite{watkins1992q} ($\alpha=0.1, \gamma=0.95$)
for $100$ episodes, $100$ steps for 9x9 grid, $500$ steps fourroom, Hanoi, and Taxi. 
We generated 8 options with each algorithms using the adjacency matrix representing the state-transition of the MDP.
Figure \ref{fig:adjacency} shows the comparison of accumulated rewards averaged over 5 runs. In all experiments, \algname{} outperformed or was on par with eigenoptions.
Figure \ref{fig:utility} shows the comparison of accumulated rewards with varying number of covering options on fourroom domain. Overall, adding more options improves performance but the added utility is diminished. It is to be expected as the target function is a concave function of the number of edges added which roughly means that the first few edges added lead to a much greater increase in algebraic connectivity than those added later on \cite{ghosh2006growing}.
Next, we evaluated the performance of the options with all states in the initiation set. Figure \ref{fig:subgoal_9x9}, \ref{fig:subgoal_fourroom} shows the comparison of accumulated rewards on fourroom and 9x9 grid domain.

\begin{figure*}[htb]
    \centering
    \centering
    \subfloat[four-room ]{\includegraphics[width=0.33\textwidth]{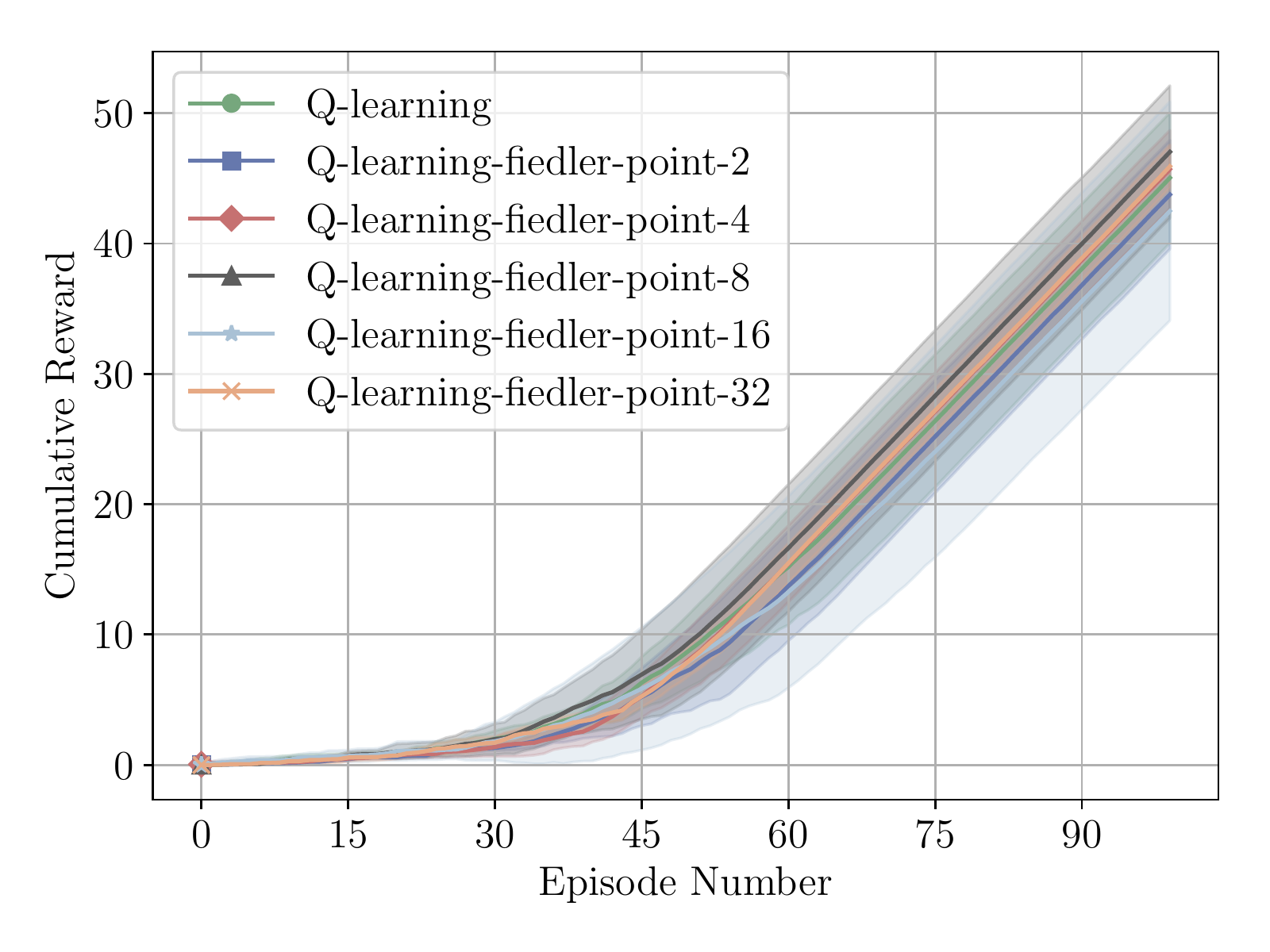}}
    \subfloat[9x9 grid ]{\includegraphics[width=0.33\textwidth]{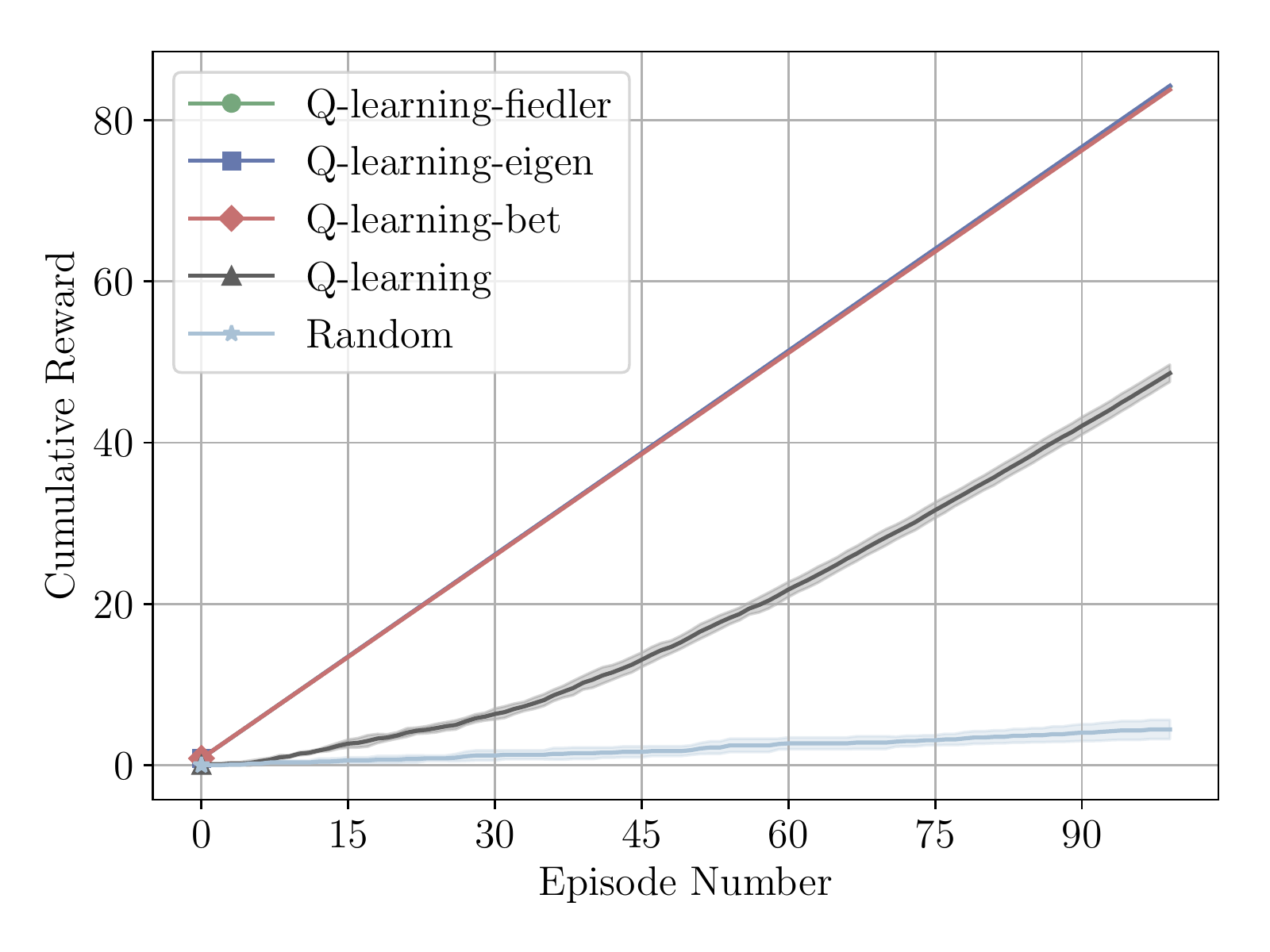} \label{fig:subgoal_9x9}}
    \subfloat[four-room ]{\includegraphics[width=0.33\textwidth]{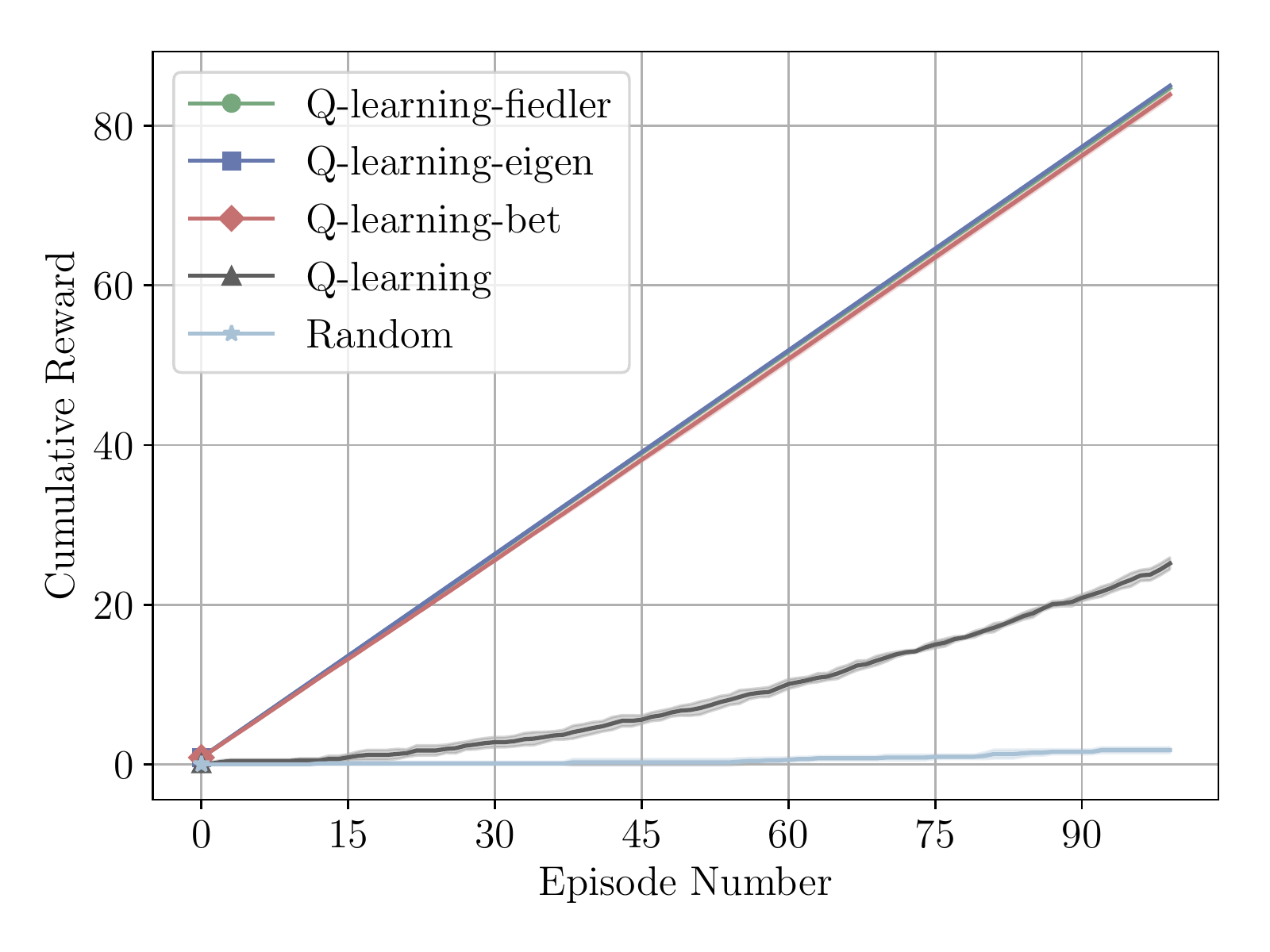} \label{fig:subgoal_fourroom}}
    \caption{(a) Comparison of performance with different number of covering options. (b), (c) Comparison of performance of options which all states are included in the initiation sets.}
    \label{fig:utility}
\end{figure*}

\subsection{Offline Approximate Option Discovery}

In the previous subsection we assumed that the agents have access to the adjacency matrix of the MDP. However, this may be difficult to achieve when the number of states is too large, as agents are not able to observe the whole state transitions in a reasonable amount of time.
Following the evaluation of \namecite{machado2017laplacian}, we evaluate our method using a sample-based approach for option discovery. Instead of giving the agent an access to the whole adjacency matrix, the agent sampled 100 trajectories of a uniform random policy to generate an incidence matrix.
We sampled each trajectory for 1000 steps for Parr's maze and 100 steps for the Race Track domain.
We feed the incidence matrix instead of the adjacency matrix to the option generation method.
As the agent has no prior knowledge on states outside the states in the incidence matrix, the agent terminates the option if it reached the subgoal state or states not in the incidence matrix. 
Other experimental settings are the same as the previous subsection.
Figure \ref{fig:adjacency} shows the resulting performance. Overall, \algname{} is outperforming or on par with eigenoptions. We have no results on betweenness options for Parr's maze as it took more than 20 minutes to generate the options.



\subsection{Online Option Discovery}

In the previous two subsections, we evaluated option discovery methods assuming that the agent has access to the state-transition prior to solving the task itself.
This assumption is reasonable in some situations such as multitask reinforcement learning where the agent is supposed to solve multiple different tasks (reward function) in the same domain (problems with the same transition function).

In this section we evaluate our method on online option discovery.
The agents generate 4 options to add to their option set every 10000 step for Parr's maze and 500 steps for the Towers of Hanoi and Taxi until the number of options reaches 32. 
We learned for 100 episodes, and episodes were 10,000 steps long for Parr's maze and 100 steps for the Towers of Hanoi and Taxi.
We used Q-learning \cite{watkins1992q} ($\alpha=0.1, \gamma=0.95$).
To compute the policy of each option, we feed the trajectories sampled by the agent so far to learn Q-values off-policy ($\alpha=0.1, \gamma=0.95$). We give an intrinsic reward of 1 to the agent when it reaches the subgoal state and ignore the rewards from the environment. 

Figure \ref{fig:online} shows the resulting performance. The agents with options are able to learn the policy faster than the agent only with primitive actions. The agents with options can reliably find the goal state even in Parr's maze whereas an agent with primitive actions is unable to find the goal.

\begin{figure*}
    \centering
	\subfloat[Parr's maze]{\includegraphics[width=0.33\textwidth]{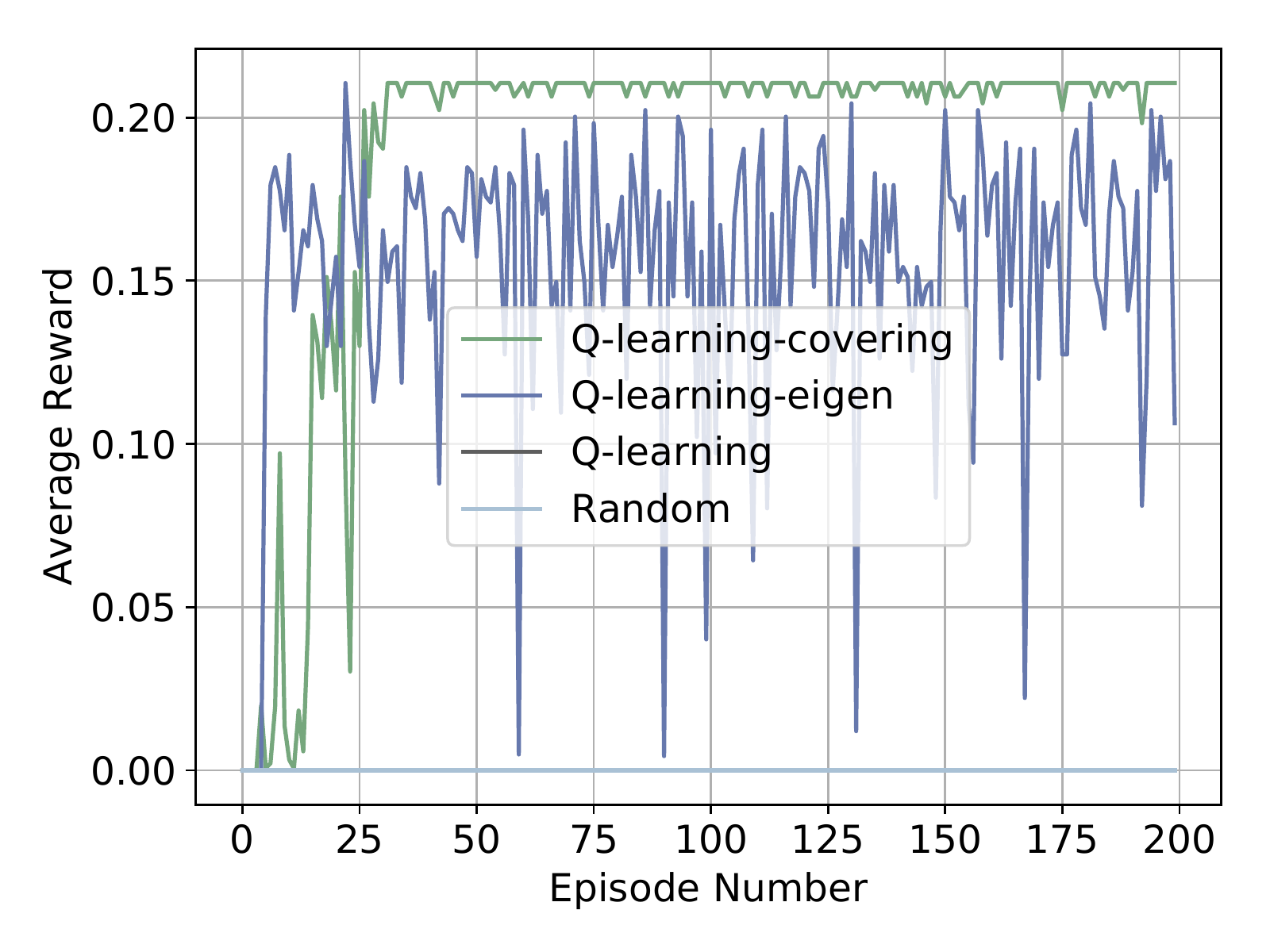}}
	\subfloat[Towers of Hanoi]{\includegraphics[width=0.33\textwidth]{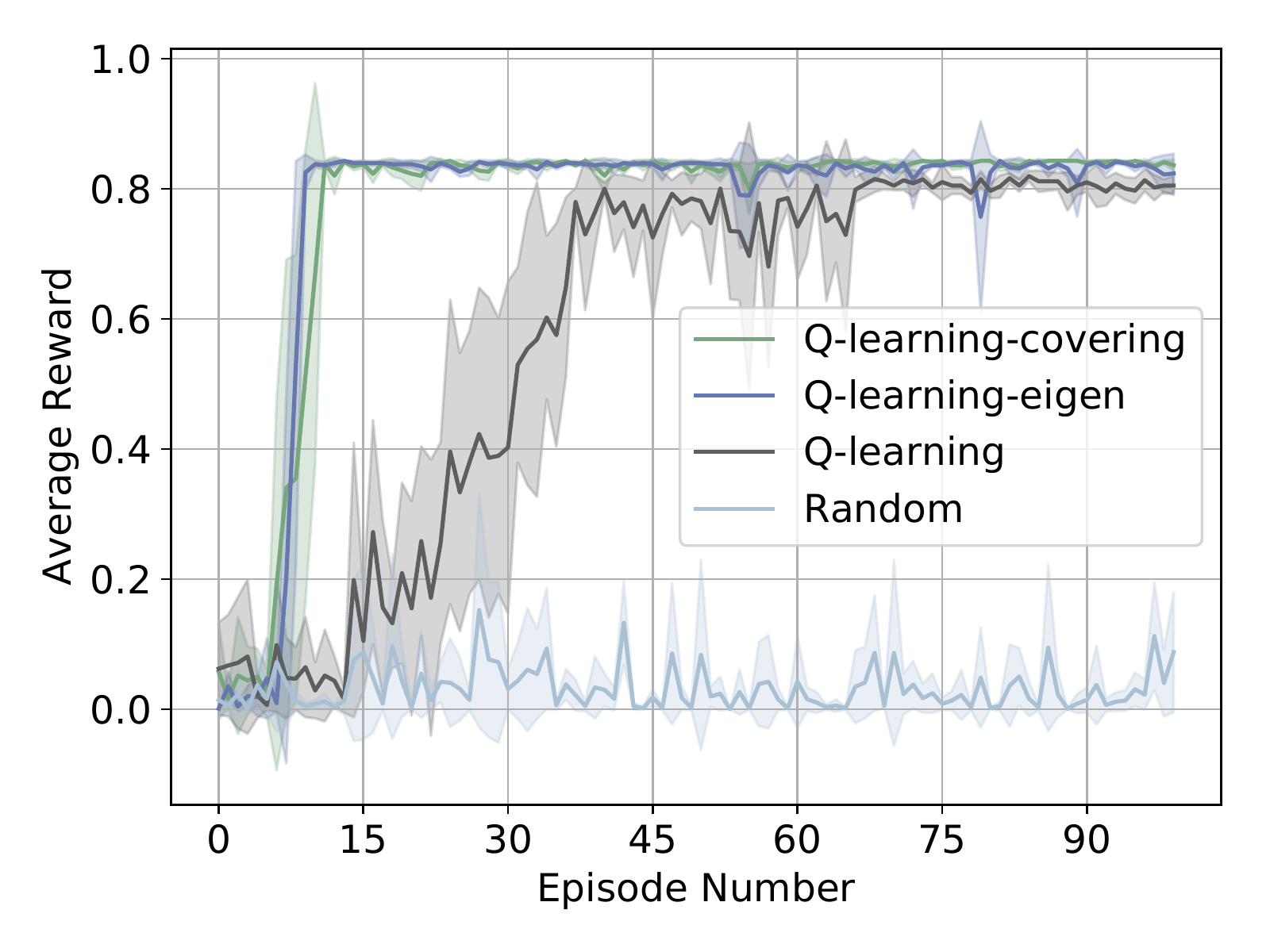}}
	\subfloat[Taxi]{\includegraphics[width=0.33\textwidth]{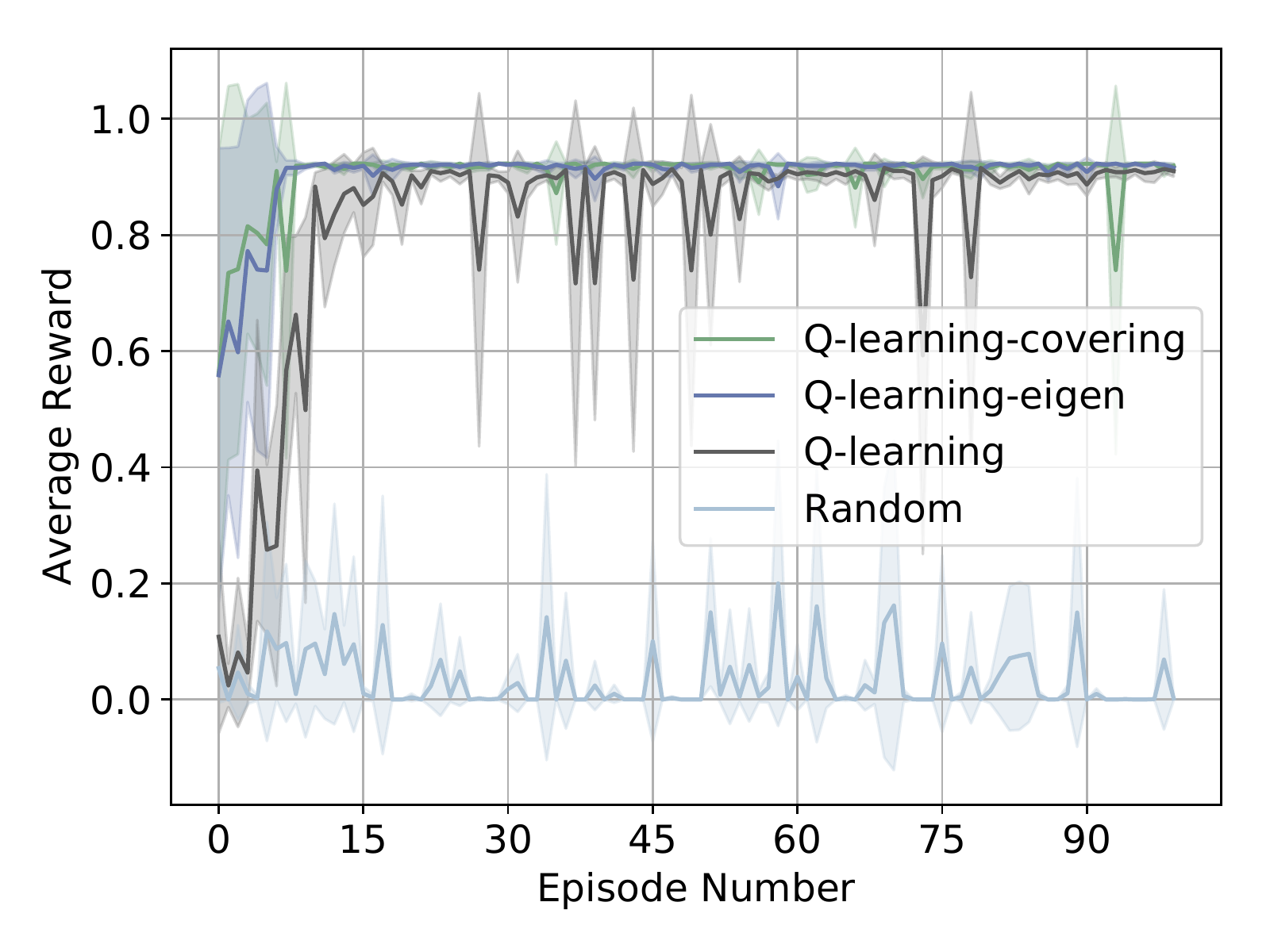}}
    \caption{Comparison of online option generation methods.}
	\label{fig:online}
\end{figure*}

\section{Related Work}
Number of methods have proposed for option discovery
\cite{iba1989heuristic,mcgovern2001automatic,menache2002q,stolle2002learning,Simsek04,Simsek2005,csimcsek2009skill,konidaris2009skill,machado2016learning,kompella2017continual,machado2017laplacian,machado2017eigenoption,eysenbach2018diversity,nair2018visual,riedmiller2018kearning}.

Many option discovery algorithms are based on informative rewards and are thus task dependent.
These methods often decompose the trajectories reaching the rewarding states into options.
Several works have proposed generating intrinsic rewards from trajectories reaching these rewarding states \cite{mcgovern2001automatic,menache2002q,konidaris2009skill}, while other approaches use gradient descent to generate options using the observed rewards \cite{mankowitz2016adaptive,bacon2017option,harb2017waiting}.

However, such approaches are often not applicable to sparse reward problems: if rewards are hard to reach using only primitive actions, options are unlikely to be discovered.
Thus, some works have investigated generating options without using reward signals.
\namecite{stolle2002learning} proposed to set states with high visitation count as subgoal states, resulting in identifying bottleneck states in the four-room domain.
\namecite{csimcsek2009skill} generalized the concept of  bottleneck states to the (shortest-path) betweenness of the graph to capture how pivotal the state is.
\namecite{menache2002q} used a learned model of the environment to run a Max-Flow/Min-Cut algorithm to the state-space graph to identify bottleneck states whereas \namecite{Simsek2005} proposed to apply spectral cut to identify bottlenecks.
These methods generate options to leverage the idea that subgoals are states visited most frequently.
On the other hand, \namecite{Simsek04} proposed to generate options to encourage exploration by generating options to relatively novel states, encouraging exploration. 

\section{Conclusions}

In this paper, we tackled the sparse reward problem by discovering options that encourage exploration.
We introduced the expected cover time which bounds the expected number of steps to reach the undiscovered rewarding state, and introduced
an option discovery method, \algname{}, which adds options that reduces the expected cover time.
We showed analytically that our method guarantees improvement of the expected cover time under certain conditions. We further conduct experiments, finding that \algname{} outperforms the previous state-of-the-art in multiple sparse reward tasks.

\bibliography{ms}
\bibliographystyle{icml2019}
\end{document}

%% file: commands.tex
\newcommand{\durl}[1]{\textcolor{blue}{\underline{\url{#1}}}}

\newcommand{\mc}[1]{\mathcal{#1}}

\newcommand{\bE}{\mathbb{E}}

\newcommand{\ra}{\rightarrow}

\DeclareMathOperator*{\argmax}{arg\,max}

\newtheorem{theorem}{Theorem}

\definecolor{dblue}{RGB}{98, 140, 190}
\definecolor{dgreen}{RGB}{113, 198, 113}
\definecolor{dpink}{RGB}{207, 166, 208}
\definecolor{dgold}{RGB}{197, 193, 170}

\newcounter{DaveDefCounter}
\newcommand{\ddef}[2]
{
\begin{mdframed}[roundcorner=1pt, backgroundcolor=white]
\vspace{1mm}
{\bf Definition \theDaveDefCounter} (#1): {\it #2}
\stepcounter{DaveDefCounter}
\end{mdframed}
}